\documentclass{article}
\usepackage{ijcai25}

\usepackage{times}
\usepackage{soul}
\usepackage{url}
\usepackage[hidelinks]{hyperref}
\usepackage[utf8]{inputenc}
\usepackage[small]{caption}
\usepackage{graphicx}
\usepackage{amsmath}
\usepackage{amsthm, xfrac}
\usepackage{booktabs}
\usepackage[switch]{lineno}

\usepackage{amsxtra, amsfonts, amssymb, amstext, mathtools}
\usepackage{nicefrac}
\usepackage{xspace}
\usepackage[dvipsnames]{xcolor}
\usepackage[linesnumbered, ruled, vlined, onelanguage]{algorithm2e}
\usepackage{wrapfig}
\usepackage{enumitem}
\usepackage[noabbrev, nameinlink]{cleveref}

\crefname{assumption}{assumption}{assumptions}
\crefname{observation}{observation}{observations}

\allowdisplaybreaks[4]
\newtheorem{theorem}{Theorem}
\newtheorem{lemma}[theorem]{Lemma}
\newtheorem{corollary}[theorem]{Corollary}
\newtheorem{definition}[theorem]{Definition}

\usepackage{subcaption}

\usepackage{tikz}
\usetikzlibrary{shapes, arrows, positioning, patterns, trees, hobby}

\newcommand{\eqnref}[1]{(\ref{#1})}
\newcommand\numberthis{\addtocounter{equation}{1}\tag{\theequation}}

\usepackage{stmaryrd}

\newcommand{\seqopt}{{\normalfont \textsc{SEQOPT}}\xspace}
\newcommand{\cliffjump}{{\normalfont \textsc{CliffJump}}\xspace}

\newcommand{\esp}{\mathrm{E}}

\DeclareMathOperator{\Hamming}{H}

\newcommand{\OI}{{\normalfont \textsc{OI}}\xspace}
\newcommand{\AM}{{\normalfont \textsc{AM}}\xspace}
\newcommand{\OW}{{\normalfont \textsc{OW}}\xspace}

\newcommand{\layer}{\mathcal{L}}%

\renewcommand*{\o}{\mathop{o}\limits}
\renewcommand*{\O}{\mathop{O}\limits}

\newcommand{\MAHH}{\textsc{MAHH}\xspace}
\newcommand{\MMAHH}{\textsc{MMAHH}\xspace}

\newcommand{\onemax}{{\normalfont \textsc{OneMax}}\xspace}

\newcommand{\cliff}{\textsc{Cliff}\xspace}

\newcommand{\jump}{\textsc{Jump}\xspace}
\newcommand{\trap}{\textsc{Trap}\xspace}

\newcommand{\R}{\ensuremath{\mathbb{R}}}

\newcommand{\N}{\ensuremath{\mathbb{N}}} 

\let\originalleft\left
\let\originalright\right
\renewcommand{\left}{\mathopen{}\mathclose\bgroup\originalleft}
\renewcommand{\right}{\aftergroup\egroup\originalright}

\title{Speeding Up Hyper-Heuristics With Markov-Chain Operator Selection and the Only-Worsening Acceptance Operator}

\author{
Abderrahim Bendahi$^1$\and
Benjamin Doerr$^2$\and
Adrien Fradin$^1$\And
Johannes F. Lutzeyer$^2$\\
\affiliations
$^1$École Polytechnique, Institut Polytechnique de Paris\\
$^2$Laboratoire d'Informatique (LIX), CNRS, École Polytechnique, Institut Polytechnique de Paris\\
\emails
\{firstname.lastname\}@polytechnique.edu
}

\begin{document}
{\sloppy

\maketitle

\begin{abstract}
The move-acceptance hyper-heuristic was recently shown to be able to leave local optima with astonishing efficiency  (Lissovoi et al., Artificial Intelligence (2023)). In this work, we propose two modifications to this algorithm that demonstrate impressive performances on a large class of benchmarks including the classic $\cliff_d$ and $\jump_m$ function classes. (i)~Instead of randomly choosing between the only-improving and any-move acceptance operator, we take this choice via a simple two-state Markov chain. This modification alone reduces the runtime on $\jump_m$ functions with gap parameter~$m$ from $\Omega(n^{2m-1})$ to $O(n^{m+1})$. (ii)~We then replace the all-moves acceptance operator with the operator that only accepts worsenings. Such a, counter-intuitive, operator has not been used before in the literature. However, our proofs show that our only-worsening operator can greatly help in leaving local optima, reducing, e.g., the runtime on Jump functions to $O(n^3 \log n)$ independent of the gap size. 
In general, we prove a remarkably good runtime of $O(n^{k+1} \log n)$ for our Markov move-acceptance hyper-heuristic on all members of a new benchmark class $\seqopt_k$, which contains a large number of functions having $k$ successive local optima, and which contains the commonly studied $\jump_m$ and $\cliff_d$ functions for $k=2$.

\end{abstract}


\section{Introduction}
\label{sec:introduction}

Selection hyper-heuristics are black-box optimization heuristics that function by combining different low-level heuristics. They were first used to solve difficult scheduling problems~\cite{CowlingKS00}, but then quickly found numerous other applications~\cite{BurkeGHKOOQ13,DrakeKOB20}.

The mathematical runtime analysis of hyper-heuristics has started around ten years ago~\cite{LehreO13}, predominantly discussing the impact of selecting different variation operators~\cite{DoerrLOW18,LissovoiOW20aaai}. More recently, hyper-heuristics having the choice between different acceptance operators were studied (although a first result can already be found in~\cite{LehreO13}). In particular, \cite{LissovoiOW19,LissovoiOW23} have shown that switching between an elitist selection (the \emph{only-improving operator \OI}) and accepting any new solution (the \emph{all-moves operator \AM}) can give excellent results. Specifically, they show that  the \emph{move-acceptance hyper-heuristic (\MAHH)} can optimize the $\cliff_d$ benchmark defined on bit-strings of length~$n$ in time $O(n^3)$, whereas comparably simple elitist evolutionary algorithm need time $\Omega(n^d)$; here $d$ is a difficulty parameter of the benchmark that can range from $2$ to $n$. A similar lower bound was shown for the Metropolis algorithm~\cite{DoerrERW23}.

However, such performance gains seem to heavily depend on the particular problem to be optimized. For the jump benchmark with difficulty parameter $m$, simple evolutionary algorithms find the optimum in expected time $O(n^m)$ \cite{DrosteJW02}, but the \MAHH needs $\Omega(n^{2m-1})$ \cite{DoerrDLS23} (for constant $m$). 

In this work, we propose two new ideas that help hyper-heuristics to leave local optima, and greatly improve their performance. We also observe that the proposed modifications resolve the difficulties detected in \cite{DoerrDLS23}. (i) From studying the proofs in \cite{DoerrDLS23}, we observe that the use of the random mixing strategy, that is, using the all-moves operator in each iteration independently with some probability~$p$, is problematic. The probability $p$ has to be small to allow for a sufficiently strong drift towards the optimum, but leaving a local optimum with radius $m$ requires $m-1$ successive uses of the AM operator, which contributes a factor of $p^{m-1}$ to the probability of successfully leaving the local optimum. To mitigate the influence of the required small rate of AM operator uses, we design a simple two-state Markov chain governing the selection of the operators. In other words, for each of the two acceptance operators, we have a switching probability. In each iteration, we use this value to decide whether we should switch to the other operator or continue with the current operator. By taking a value such as $1/2$ for the probability of switching away from the AM operator, longer stretches of using this operator become more likely, which eases the leaving of local optima with larger basins of attraction. We call the resulting algorithm \emph{Markov move-acceptance hyper-heuristic (\MMAHH)}. As an example of the usefulness of this approach, we show that the \MMAHH choosing the two operators \OI and \AM (with the long-term rates of the operators as in previous works) optimizes \jump functions in time $O(n^{m+1})$, a considerable speed-up from $\Omega(n^{2m-1})$. We note that in general this way of choosing between operators has been used before, but to the best of our knowledge no mathematical runtime result has been shown, and it has not been used in conjunction with acceptance operators until now. 

We then propose the only-worsening (\OW) acceptance operator. It accepts the new solution only if it is strictly worse than the parent. While this operator contradicts the idea of incremental optimization, in our context it is less counter-intuitive than what appears at first. We recall that the main working principle of the \AM operator exploited in previous works is that it allows the algorithm to leave local optima. For this aim, searching for inferior solutions is, in fact, a logical approach. This intuitive consideration is supported by our mathematical runtime analysis, which in particular shows that the \MMAHH with the two acceptance operators \OI and \OW optimizes any \jump and \cliff function in expected time $O(n^3 \log n)$. We extend this result to a new benchmark called $\seqopt_k$, which consists of a broad class of pseudo-Boolean functions having $k$ local optima that in particular includes the classic \onemax, \jump, \cliff, and \trap benchmarks. We show that our \MMAHH optimizes any function in $\seqopt_k$ in expected time $O(n^{k+1} \log n)$.

With this work, we introduce two new ideas for the design of effective move-acceptance hyper-heuristics. We believe that in particular, the only-worsening operator to leave local optima, could give rise to further theoretical study and possibly the design of new benchmarks designed to test its limits. We furthermore hope that future consideration of our $\seqopt_k$ benchmark enables a joint observation of the effectiveness of hyper-heuristics that subsumes results on the different commonly studied \jump and \cliff benchmarks.

\section{Preliminaries}

In this section, we briefly recall the classic benchmark problems relevant for this work, define our new benchmark $\seqopt_k$, define our Markov move-acceptance hyper-heuristic, the only-worsening acceptance operator, and provide the tools to analyze our hyper-heuristics.

\subsection{Benchmarks}\label{subsec:benchmarks}

As standard in the theory of randomized search heuristics~\cite{NeumannW10,AugerD11,Jansen13,ZhouYQ19,DoerrN20}, we regard pseudo-Boolean optimization problems, that is, we aim at maximizing functions $f$ that map bit-strings $x \in \{0, 1\}^n$ with a fixed positive length $n \in \N_{>0} = \{1, 2, 3, \ldots\}$ to a numerical value $f(x) \in \R$. When using asymptotic notation, this shall always be for $n \to \infty$.

Well-known examples of such functions in the theory literature include the following \onemax, \trap, $\cliff_d$ and $\jump_m$ benchmarks.



For a bit-string $x = (x_1, \ldots, x_n) \in \{0, 1\}^n$, let $\| x \|_1 = x_1 + x_2 + \cdots + x_n$ denote the number of ones in $x$. The following functions are standard benchmarks in the theory of randomized search heuristics.
    \begin{align*}
        & \onemax \colon  x \mapsto \| x \|_1 
        ; \\
        & \cliff_d \colon  x \mapsto \begin{cases}
            \| x \|_1, & \text{if $\| x \|_1 \leq n - d$;} \\ 
            \| x \|_1 - d + \frac{1}{2}, & \text{otherwise;}
        \end{cases} \\ 
        & \jump_m \colon  x \mapsto \begin{cases}
            m + \| x \|_1, & \text{if $\| x \|_1 \in [0..n-m] \cup\{n\}$;} \\ 
            n - \| x \|_1, & \text{otherwise.}
        \end{cases}
    \end{align*}
Here $d \in [1..n-1]$ and $m \in [1..n]$ are difficulty parameters of the \cliff and \jump benchmark. The completely deceptive function $\jump_n$ is also called $\trap$. All these functions have $x^* = (1, \dots, 1)$ as unique global optimum (maximum), and all are \emph{functions of unitation}, that is, the objective value of a solution depends only on the number of ones. This motivates the definition of the \emph{$k$-th layer} as
    \[ \layer_k := \{ x \in \{0, 1\}^n \mid \Hamming(x, x^*) = n - \|x\|_1 = k\} \]
for $k \in [0..n]$, where we used $\Hamming(\cdot, \cdot)$ to denote the Hamming distance of two bit-strings. Note that $\layer_k$ is the set of all bit-strings at distance $k$ from the global maximum $x^*$, that is, the numbering starts at the global optimum.

All benchmarks above also have the property that they are composed of intervals of layers in which the function is only increasing or only decreasing; for $\jump_m$ and $\cliff_d$ these intervals have lengths $2$, $m-1$, and $n-m+1$. In the following two definitions, we extend this property to arbitrary interval numbers and lengths and obtain the very general benchmark $\seqopt_k$, having $k$ successive local minima and maxima.




\begin{definition}[Monotonicity across layers]\label{def:monotonicity-layers}
    Let $h \in [0..n-1]$ and $f \colon \{0, 1\}^n \to \R$. 
    We say that $f$ is increasing (resp.\ decreasing) between layers $\layer_{h + 1}$ and $\layer_h$ 
    if for any $x \in \layer_h$ and $y \in \layer_{h + 1}$ we have
        \[ f(x) > f(y) \text{ (resp. $f(x) < f(y)$)}.\]
    We denote this by $\layer_h \overset{f}{\succ} \layer_{h + 1}$ (resp. $\layer_h \overset{f}{\prec} \layer_{h + 1}$).
\end{definition}




\begin{definition}[The \seqopt benchmark]\label{def:seqopt}
    Let $n \geq 2$, $k \in [ 0 .. n - 2 ]$ and $d_0 = n  > d_1 > d_2 > \dots > d_k > d_{k+1} = 0$ be integers. We define {\normalfont $\seqopt_k(d_1, \ldots, d_k)$} to be the set of all functions $f \colon \{0, 1\}^n \to \R$ such that
    \begin{enumerate}
        \item $x^* = (1, \dots, 1)$ is the unique global maximum of $f$,
        \item for any $\ell \in [0 .. k]$, if $k - \ell$ is even then
            \[ \layer_{d_{\ell}} \overset{f}{\prec} \cdots \overset{f}{\prec} \layer_{d_{\ell + 1}}, \]
        and if $k - \ell$ is odd, $f$ satisfies
            \[ \layer_{d_{\ell}} \overset{f}{\succ} \cdots \overset{f}{\succ} \layer_{d_{\ell + 1}}. \]

        The union of these classes of functions, for fixed $k$, will be denoted by
            \[ \seqopt_k = \bigcup_{n > d_1 > \cdots > d_k > 0} \seqopt_k(d_1, \ldots, d_k).\]
    \end{enumerate}
\end{definition}



Note 
that we have $\onemax \in \seqopt_0$, $\trap \in \seqopt_1(1)$, $\cliff_d \in \seqopt_2(d, d - 1)$ and, for $m < n$, $\jump_m \in \seqopt_2(m, 1)$.

\subsection{The Markov Move-Acceptance Hyper-Heuristic}\label{subsec:mmahh}

We now introduce a novel algorithm called the \emph{Markov Move-Acceptance Hyper-Heuristic} algorithm (\MMAHH). We recall that the Move-Acceptance Hyper-Heuristic algorithm (\MAHH) first proposed in \cite{LehreO13} and then intensively studied in \cite{LissovoiOW23} is a simple randomized local search heuristic which randomly mixes between the \OI and \AM operators, that is, in each iteration independently is chooses between the \AM operator (with some, usually small, probability~$p$) and the \OI operator (with probability $1-p$).  

From studying the proof of the unsatisfactory $\Omega(n^{2m-1})$ runtime bound for this algorithm on the $\jump_m$ benchmark, we learn that the main reason for this negative performance is the low probability of having $m-1$ consecutive uses of the \AM operator, which stems from the independent choice of the operators. To allow for longer phases of using the same operator, our \MMAHH leverages a simple $2$-state Markov chain to govern the selection of the acceptance operators. In other words, if the current operator is \OI, this operator is kept for the next iteration with probability $1-p$, but changed to \AM with probability $p$. The switching probability from \AM to \OI is denoted by~$q$. See Algorithm~\ref{alg:mmahh} for the pseudocode of the \MMAHH, where the $\textsc{Markov}$ operator refers to sampling the Markov chain illustrated in Figure~\ref{fig:markovchain}.

We also introduce a new acceptance operator, only-worsening (\OW), to substitute the \textsc{AllMoves} (\AM) operator. This new operator works in the same fashion as the well-known \textsc{Only\-Improving} (\OI) acceptance operator except that \OW only accepts worsening moves, i.e., moves decreasing the function value. The idea of this, counter-intuitive, operator is to speed-up leaving local optima. When studying the previous runtime analyses for the \MAHH, we see that they profit from the \AM operator in that it allows the algorithm to leave local optima. If this is the target, then the \OW operator should be even better suited, and this is what we will observe in this work.


\begin{figure}[t]
    \hypertarget{fig1}{}
    \centering
    \begin{tikzpicture}[->, >=stealth', auto, node distance = 4cm]
        \node[circle, minimum size = 1.2cm, draw](0) at (0, 0) {\OI};

        \node[circle, minimum size = 1.2cm, draw](1) at (3, 0) {\OW};

        \path (0) edge [loop left] node [left] {$1 - p$} (0);
        \path (1) edge [loop right] node [right] {$1 - q$} (1);

        \path (0) edge [bend right] node [below] {$p$} (1);
        \path (1) edge [bend right] node [above] {$q$} (0);

    \end{tikzpicture}
    \caption{\emph{Transition probabilities between the two operators of the \MMAHH, here \OI and \OW}.}
    \label{fig:markovchain}
\end{figure}
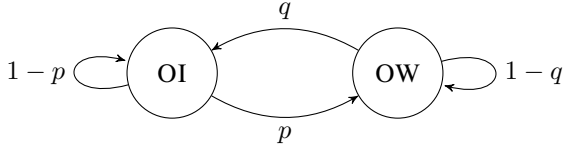


\subsection{Notation for the Analysis of the \MMAHH}

Throughout this work, we denote by $x_t$, $y_t$ and $s_t$, respectively, the current solution at time~$t$, its Hamming distance to $x^*$ (i.e., $y_t = \Hamming(x_t, x^*)$), and the move-acceptance operator used at time~$t$. We initialize with a random solution $x_0 \sim \mathcal{U}(\{0,1\}^n)$ and $s_0 = \OI$. We always denote by $f \colon \{0,1\}^n \to \R$ the function to be maximized.
Moreover, the transition probabilities from \OI to \OW and \OW to \OI will be denoted respectively by $p$ and $q$ with $0 < p, q < 1$.
 
We call $T := \inf \{ t \geq 0 \mid x_t = x^* \}$ the runtime, that is,  the first time we reach the global maximum. Let $T_s^{(k)}$ denotes the $k$-th switching time between the operators. Hence  $T_s^{(0)} = 0$ and for any $k \geq 0$, we have $T_s^{(k + 1)} = \inf \{ t \geq T_s^{(k)} \mid s_t \neq s_{T_s^{(k)}} \}$. We also define $Z_k  = T_s^{(k + 1)} - T_s^{(k)}$ to be the length of the $k$-th \emph{phase}, where a \emph{phase} is a (maximum) time interval in which the same operator is used.


This notation is summarized in \Cref{tab:notation} and illustrated in \Cref{fig:drawing}.





\begin{table}
  \caption{Table of Notation.}
  \label{tab:notation}
  \begin{tabular}{cl}
    \toprule
    \textbf{Symbol} & \textbf{Meaning} \\
    \midrule
    $T$ & The first hitting time of $x^*$ \\
    $T_s^{(k)}$ & The starting time of the $k$-th phase \\
    $p, q$ & Transition probabilities of the Markov chain \\
    $Z_k$ & The length of the $k$-th phase \\
    $f$ & A function with a unique global maximum \\ 
    $x^*$ & The global maximum of $f$, $x^* = \{1\}^n$ \\
    $x_t$ & The bit-string during the $t$-th iteration \\
    $y_t$ & The Hamming distance to $x^*$ at time $t$ \\
    $s_t$ & The move-acceptance operator at time $t$\\
    $x_0$ &  The initial bit-string, $x_0 \sim \mathcal{U}(\{0,1\}^n)$ \\
    $s_0$ & The initial move-acceptance operator, $s_0 = \OI$\\
  \bottomrule
\end{tabular}
\end{table}

\begin{figure}
        \hypertarget{fig2}{}
        \centering
        \resizebox{\columnwidth}{!}{%
            \begin{tikzpicture}
                \def\gap{0}; \def\void{3};
                \def\x{0};
                \def\la{2.5}; \def\lb{4.5}; \def\lc{1.5}; \def\ld{3.5};
        
                \node (A1) at (\x, 0) {}; \node (A2) at (\x+\la, 0) {};
                \node (B1) at (\x+\la+\gap/2, 0) {}; \node (B2) at (\x+\la+\lb+\gap/2, 0) {};
                \node (C1) at (\x+\la+\lb+2*\gap/2, 0) {}; \node (C2) at (\x+\la+\lb+\lc+2*\gap/2, 0) {};
                \node (D1) at (\x+\la+\lb+3*\gap/2+\void, 0) {}; \node (D2) at (\x+\la+\lb+\lc+\ld+3*\gap/2+\void, 0) {};
        
                \node[above right = 1cm and -0.35cm of A1, red] (T0) {$T_s^{(0)}$};
                \node[above right = 1cm and -0.35cm of B1, blue] (T0) {$T_s^{(1)}$};
                \node[above right = 1cm and -0.35cm of C1, red] (T0) {$T_s^{(2)}$};
                \node[above right = 1cm and -0.35cm of C2, blue] (T0) {$T_s^{(3)}$};
                \node[above right = 1cm and -0.35cm of D1] (T0) {$T_s^{(k^* - 1)}$};
                \node[above right = 1cm and -0.35cm of D2] (T0) {$T_s^{(k^*)}$};
    
                \node[above left = 0.1cm and 0.6cm of D2, purple] (T) {$T$};
                \draw[thick, purple] (T.center) ++(0, -0.35) -- ++(0, -0.4);
                
                \draw[<->, red] (A1) -- (A2) node[midway, below = 0.2cm] {$Z_0$} node[midway, above] {$s_0$};
                \draw[<->, blue] (B1) -- (B2) node[midway, below = 0.2cm] {$Z_1$} node[midway, above] {$s_1$};
                \draw[<->, red] (C1) -- (C2) node[midway, below = 0.2cm] {$Z_2$} node[midway, above] {$s_2$};
                \path (C2) -- node[auto=false]{\ldots} (D1);
                \path (D2) -- node[auto=false]{\ldots} ++(1, 0);
                \draw[<->] (D1) -- (D2) node[midway, below = 0.2cm] {$Z_{k^* - 1}$} node[midway, above] {$s_{k^* - 1}$};
        
                \draw[thick, dashed, red] (A1.center) ++(0, 1) -- ++(0, -2);
                \draw[thick, dashed, blue] (B1.center) ++(0, 1) -- ++(0, -2);
                \draw[thick, dashed, red] (C1.center) ++(0, 1) -- ++(0, -2);
                \draw[thick, dashed, blue] (C2.center) ++(0, 1) -- ++(0, -2);
                \draw[thick, dashed, gray] (D1.center) ++(0, 1) -- ++(0, -2);
                \draw[thick, dashed, gray] (D2.center) ++(0, 1) -- ++(0, -2);
            \end{tikzpicture}
        }%
        \caption{\emph{A generic setting, depicting the phases (red or blue), their length $(Z_k)_{k \in \N}$, the switching times $( T_s^{(k)})_{k \in \N}$ and the stopping time $T$ occurring during the phase $k^*$}.}
        \label{fig:drawing}
    \end{figure}
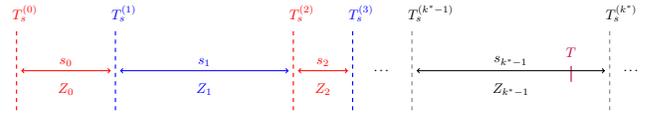

\begin{algorithm}
\DontPrintSemicolon
\caption{{The Markov move-acceptance hyper-heuristic with the acceptance operators \OI and \OW.}}\label{alg:mmahh}
\SetKwProg{Init}{Initialization}{:}{}
\Init{}{
    $t \gets 0$\;
    $x_0 \sim \mathcal{U}(\{0, 1\}^n)$, a uniformly sampled bit-string\;
    $s_0 \gets \OI$\;
}

\vspace{.5\baselineskip}
\While{true}{
    $x' \gets \textsc{RamdomOneBitFlip}(x_t)$\;
    \uIf{$s_t = \OW$ \textbf{and} $f(x') < f(x_t)$}{
        $x_{t + 1} \gets x'$\;
    }\ElseIf{$s_t = \OI$ \textbf{and} $f(x_t) < f(x')$}{
        $x_{t + 1} \gets x'$\;
    }
    $s_{t + 1} \gets \textsc{Markov}(s_t)$\; \label{alg:mmahh-markov}
    $t \gets t + 1$\;
}
\end{algorithm}

\subsection{Tools for the Analysis of the \MMAHH
}\label{subsec:technical-lem}

We now develop the tools we need to analyze the \MMAHH. 
Unfortunately, the generality of the Markov chain of our setup disallows the use of the methods previously employed in the analysis of the \MAHH, which were mostly drift arguments based on the fitness of the current solution. Nonetheless, with some mathematical effort, we manage to obtain very precise estimates of the quantities of interest.





\paragraph{Probability of improvement in one phase.} 
We start by computing the probabilities that a single phase of \OI usages leads to a certain fitness improvement. From the way we constructed the Markov chain, the length $Z_k$ of each phase is independent from the solution $x_{T_s^{(k)}}$ it starts with (and hence its fitness $y_{T_s^{(k)}}$), and follows a geometric law.

For $0 \leq h, k \leq n$, let $p_k^h = \Pr[y_{T_s^{(1)}} \leq h \mid y_0 = k, s_0 = \OI]$ denote the probability to reach $\layer_h$ in one \OI phase when starting in state \OI in $\layer_k$, assuming that we optimize the \onemax benchmark. The following computation of these probabilities is a cornerstone of our analysis.
\begin{lemma}\label{lem:oi-prob}
    For all $0 < p < 1$ and  $0 \leq h, k \leq n$, we have
    \begin{align*}
        & p_k^h = \begin{cases}1, & \text{if $k \leq h$;} \\ \frac{1}{1 - p} \frac{\Gamma(k + 1) \Gamma\left( \frac{n p}{1 - p} + h + 1 \right)}{\Gamma(h + 1) \Gamma\left( \frac{n p}{1 - p} + k + 1 \right)}, & \text{otherwise.} \end{cases} 
    \end{align*}
\end{lemma}

The next key lemma shows that $p = \Theta(\frac{1}{n \log(n)})$, that is, a quasi-linear length of the \OI phase, suffices to have a constant probability $p^0_n$ to optimize \onemax in one \OI phase even when starting in the all-zero string.

\begin{lemma}\label{lem:limit}
    Let $c > 0$ be a constant and let $p = \frac{1}{c n \log(n)}$. Then $p^0_n = \left( 1 + o(1) \right)e^{-1 / c}$. 
\end{lemma}


From a simple domination argument, exploiting that our hyper-heuristics have to visit all intermediate levels, we immediately obtain the following minimality statement.
\begin{lemma}[Minimality of $p^0_n$]\label{lem:minimality-p^0_n}
For all $0 \leq h, k \leq n$, we have $p_k^h \geq p_n^0$.
\end{lemma}
A formal proof of this lemma is deferred to \Cref{app:tools-analysis-mmahh}.


As \OI and \OW operators work in a symmetric fashion, similar results hold on \OW and decreasing moves. More specifically, given $0 \leq h, k \leq n$, we have
    \begin{align*} 
        \Pr[y_{T_s^{(1)}} &\, \geq n - h \mid y_0 = n - k, s_0 = \OW] \\ 
        & = \begin{cases} 1, & \text{if $k \leq h$;} \\ \frac{1}{1 - q} \frac{\Gamma(k + 1)\Gamma\left( \frac{n q}{1 - q} + h + 1 \right)}{\Gamma(h + 1)\Gamma\left( \frac{n q}{1 - q} + k + 1 \right)}, & \text{otherwise.} \end{cases}
    \end{align*}
In particular, when $q = \Theta(\frac{1}{n \log(n)})$, the probability to climb down \onemax entirely in one phase of \OW is $\Omega(1)$.

We can further extend \Cref{lem:oi-prob} to the \seqopt benchmark. Note that for any $k$ and $f \in \seqopt_k$, for each $\ell$ the transition probability between layer $\layer_\ell$ and $\layer_{\ell - 1}$ (resp. $\layer_\ell$ and $\layer_{\ell + 1}$, when defined), are the same as for \onemax, namely, $\frac{\ell}{n}$ (resp. $\frac{n - \ell}{n}$). Thus, the computations done in \Cref{lem:oi-prob} still hold for $f$ in regions where $f$ is monotonic across the layers, as defined in \Cref{def:monotonicity-layers}.

We further note that, for \onemax again, the average number of pairs of phases of \OI followed by \OW needed to reach the optimum is bounded from above by $\frac{1}{p_n^0}$, since in the worst scenario, every phase of \OW bring us back to the all-zero string. Extending this insight again to \seqopt, we see that in the quasi-linear regime, reaching a neighboring local optimum starting from a local optimum of any $\seqopt$ function only takes $O(1)$ pairs of phases. This particular fact will be referred as to the \emph{one-phase approximation}, stated more formally in the next lemma.
\begin{lemma}[One-phase approximation]\label{lem:one-phase-approx}
    Let $p = \Theta(\frac{1}{n\log(n)})$ and $q = \Theta(\frac{1}{n\log(n)})$. Let $f \in \seqopt_k(d_1, \ldots, d_k)$, $\ell \in [0..k]$ and $x_0 \in \layer_{d_{\ell}}$. Then the {\normalfont \MMAHH} algorithm reaches some $x \in \layer_{d_{\ell - 1}} \cup \layer_{d_{\ell + 1}}$ in $O(1)$ phases on average.
\end{lemma}

\section{\MMAHH With OI+AM on Jump}\label{subsec:mmahh-oi-am-jump}
The core objective of this section is to prove that the sole use of the Markov chain improves significantly the performance of the \MAHH on $\jump_m$, reducing its average runtime from $\Omega\left(\frac{n^{2 m - 1}}{(2 m - 1)!}\right)$ to $O(n^{m + 1})$ as outlined in the next theorem.

To clearly point out the contribution of the Markov chain, we consider the \MMAHH on $\jump_m$ where $1 < m < \frac{n}{2}$ using the \OI and \AM operators, with probability $p$ and $q$ to choose them respectively (we just replace \OW in \Cref{fig:markovchain} by \AM).
\begin{theorem}[Runtime analysis of the \MMAHH on $\jump_m$ with \OI-\AM]\label{thm:mmahh-time-jump-oi-am}
   The time $T$ taken by the {\normalfont \MMAHH} on {\normalfont $\jump_m$} with $1 < m < \frac{n}{2}$ to reach $x^*$, using \OI and \AM, satisfies
        \begin{align*}
            \esp[T] = \O\left( (1 + pn) \left( \frac{1}{p} + \frac{1}{q} \right) N_{n, m, q} \right),
        \end{align*}
    where $N_{n, m, q} = n + \frac{n^m}{(m - 1)! (1 - q)^{m - 2}}$, $n \geq 3$ and $p$, $q$ are such that $\frac{m}{2 (n - 2 m)}\left( q + \frac{4}{n} \right) \geq p$.
    
    Notably, it suffices to have $\frac{m}{2 n} q \geq p$ from where we have:
    \begin{enumerate}
        \item When $p = \frac{m}{c n}q$ for some $c \geq 2$ and $q = \frac{1}{2}$ then
        \[ \esp[T] = \O\left( n^2 + \frac{2^{m - 2} n^{m + 1}}{(m - 1)!} \right) = O(n^{m + 1}). \]

        \item When $p = \frac{m}{c n}q$ and $q = \frac{1}{d m}$ where $c \geq 2$ and $d \geq 1$ then
        \[ \esp[T] = \O\left( n^2 + \frac{e^{1/d}\, n^{m + 1}}{(m - 1)!} \right) = O\left( \frac{n^{m + 1}}{(m - 1)!} \right). \]
    \end{enumerate}

\end{theorem}

The proof of the theorem relies mainly on drift theorems. Hence we start by studying the drift in the following lemmas. Throughout this part, the drift over a phase of \AM (resp. \OI) starting in $i \in [0..n]$ (denoting the number of zero bits) will be $\Delta^{\AM}_i$ (resp. $\Delta^{\OI}_i$) and defined as
    \[ \Delta^{\AM}_i = \esp\left[ y_{T_s^{(2k)}} - y_{T_s^{(2 k + 1)}} \mid \, y_{T_s^{(2k)}} = i, s_{T_s^{(2k)}} = \AM \right]. \]

Lemmas \ref{lem:driftAM}, \ref{lem:driftOI} and \ref{lem:driftAM-OI} give the drift over a phase of \AM, a phase of \OI and a pair \AM\!+\OI of phases.
\begin{lemma}[Drift over a phase of \AM]\label{lem:driftAM}
Let $i \in [0..n]$ and an integer $k \geq 0$, the drift over a phase of \AM on \onemax is
\begin{align*}
    \esp[y_{T_s^{(2k)}} - y_{T_s^{(2 k + 1)}} & \mid \, y_{T_s^{(2k)}} = i, s_{T_s^{(2k)}} = \AM ] \\ &= \frac{2 i - n}{2 + q(n-2)}.
\end{align*}

\end{lemma}

\begin{lemma}[Drift over a phase of \OI]\label{lem:driftOI}
Let $i \in [0..n]$ and an integer $k \geq 0$, the drift over a phase of \OI on \onemax is
\begin{align*}
    \esp[y_{T_s^{(2k)}} - y_{T_s^{(2 k + 1)}} & \mid \, y_{T_s^{(2k)}} = i, s_{T_s^{(2k)}} = \OI ] \\
    &= \frac{i}{1 + p (n - 1)}.
\end{align*}
\end{lemma}

\begin{lemma}[Drift over a pair \AM\!+\OI of phases]\label{lem:driftAM-OI}
For any $i \in [0..n]$ and $k \geq 0$ an integer, the drift over a phase of \AM followed by a phase of \OI on \onemax is
    \begin{align*}
       \esp[y_{T_s^{(2k)}} - y_{T_s^{(2 k + 2)}} & \mid y_{T_s^{(2k)}} = i, s_{T_s^{(2k)}} = \AM] \\
       &= \Delta^{\AM}_i + \Delta^{\OI}_{\left( i - \Delta^{\AM}_i \right)}, 
    \end{align*}
i.e., we can split the drift across two phases and plug the average position after a phase of \AM directly in the drift of a phase of \OI.
\end{lemma}

Now, the following lemma provides an upper bound on the average number of phases of \AM needed to reach the global maximum at $x^* = \{1\}^n$ from a local maximum in layer $\layer_m$.
\begin{lemma}\label{lem:avg-phases-AM}
    The probability to reach the global maximum $x^*$ during a single phase of \AM starting from a local maximum in layer $\layer_m$ is
        \[ \Pr[x_{T_s^{(1)}} = x^* \mid x_0 \in \layer_m, s_0 = \AM] \geq (1 - q)^{m - 2} \frac{m!}{n^m}. \]
\end{lemma}



Let 
    \[ X^* = \{ x \in \{0, 1\}^n \mid \| x\|_1 \in \{n - m, n\} \}, \]
the set of (local and global) maxima of $\jump_m$ and we consider the potential
    \[ d \colon x \mapsto \begin{cases}
        \lvert n - m - \| x \|_1\rvert, & \text{if $x \neq x^*$;} \\ 
        0, & \text{otherwise.}
    \end{cases} \]

In the next two lemmas, we work out lower bounds on the drift of the bit-string sequence, distinguishing two cases based on the landscape of the $\jump_m$ function. In \Cref{lem:drift-gap-region}, we lower bound the drift in the gap region, i.e., at positions $x \in \{0,1\}^n$ such that $n - m < \| x \|_1 < n$, directly on the sequence $(x_t)_{t \geq 0}$ using the potential $d$ defined earlier. This result will provide an upper bound on the average time needed to climb towards $x^*$ from a point $x_0$ in the gap region. On the other hand, \Cref{lem:drift-left-slope} deals with the drift in the region with the \onemax-like slope, i.e., bit-strings $x \in \{0,1\}^n$ for which $0 \leq \| x \|_1 < n - m$. This time, we work at the scale of a pair of phases \AM\!+\OI since locally, depending on the current operator \OI or \AM, the drift might have opposite signs, notably in the region $[\frac{n}{2}.. (n - m)]$. To this aim, we will use \Cref{lem:driftAM-OI} along with the additive drift theorem with overshooting \cite{KotzingK19} to upper bound the expected time to climb this left slope from an initial point $x_0$.


\begin{lemma}[Drift in the gap region]\label{lem:drift-gap-region}
    Let $x \in \{0,1\}^n$ with $n - m < \| x \|_1 < n$. Then
        \[ \esp[d(x_t) - d(x_{t + 1}) \mid x_t = x] \geq \frac{2 d(x)}{n}. \]
\end{lemma}

\begin{lemma}[Average time spent in the slope towards a local maximum]\label{lem:drift-left-slope}
    Let $x_0 \in \{0, 1\}^n$ such that $\| x_0 \|_1 < n - m$ and $T_0 = \inf\{ t \geq 0 \mid \| x_t \|_1 = n - m \}$ be the time taken by the {\normalfont \MMAHH} to reach a local maximum of {\normalfont $\jump_m$}, starting in $(x_0, s)$ with $s \in \{\OI, \AM\}$. Then
        \[ \esp[T_0] = O \left( (n - \| x_0 \|_1 ) (1 + p n) \left( \frac{1}{p} + \frac{1}{q} \right) \right), \]
    provided that $\frac{m}{2 (n - 2 m)}\left( q + \frac{4}{n} \right) \geq p$ and $n \geq 3$.

    In particular, if $\| x_0 \|_1 = n - m - 1$, then
        \[\esp\left[ T_0 \right] = \O\left( m (1 + pn) \left(\frac{1}{p} + \frac{1}{q} \right) \right) .\]
\end{lemma}

\begin{lemma}[Average time to climb towards a maximum]\label{lem:avg-time-to-climb-jump}
    Let $T_1 = \inf\{ t \geq 0 \mid x_t \in X^* \}$ be the time taken by the {\normalfont \MMAHH} starting with $(x_0, \OI)$ to reach $x^*$ or a local maximum of {\normalfont $\jump_m$} then
        \[ \esp[T_1] = \O\left( n (1 + p n) \left( \frac{1}{p} + \frac{1}{q} \right) \right), \]
    provided that $\frac{m}{2 (n - 2 m)}\left( q + \frac{4}{n} \right) \geq p$ and $n \geq 3$.
\end{lemma}


We can now prove the main theorem of this section.
\begin{proof}[Proof of \Cref{thm:mmahh-time-jump-oi-am}]
    The overall runtime of the \MMAHH can be split in two times $T_1$ and $T_2$ such that $T = T_1 + T_2$,
        \[ T_1 = \inf\{ t \geq 0 \mid x_t \in X^* \}, \]
    and $T_2$ is the time, starting in a local maximum of $\jump_m$ (if it happens), to reach $x^*$.
    First, by \Cref{lem:avg-time-to-climb-jump} we have
        \[  \esp[T_1] = \O\left( n (1 + pn) \left(\frac{1}{p} + \frac{1}{q} \right) \right), \]
    and it remains to upper bound $\esp[T_2]$. Of course, $T_2 = 0$ if we already reached $x^*$ during the first phase. Now, assume we do not, and hence, we are in some local maximum of $\jump_m$, i.e., some $x \in \{0, 1\}^n$ such that $\| x \|_1 = n - m$. We then define excursions that start upon leaving layer $\layer_m$ (the set of local maxima of $\jump_m$) and end either when we come back to this set of local maxima in $\layer_m$ (in case of a failure) or when we reach $x^*$. As every excursion starts in state \AM, the number $N$ of such excursions can be upper bounded by the number $N^*$ of phases of \AM needed to reach $x^*$ from layer $\layer_m$. By \Cref{lem:avg-phases-AM}, this can be bounded by
        \[ \esp[N] \leq \esp[N^*] \leq \frac{n^m}{m! (1 - q)^{m - 2}}. \]

    Then, if we denote by $e_i$ the $i$-th excursion and by $\lambda(e_i)$ its length, , i.e., the time of
the excursion $e_i,$ we can write
        \[ T_2 = T^w_0 + \sum_{i = 1}^N \left( \lambda(e_i) + T^w_i \right), \]
    where $T^w_i$ for $0 \leq i \leq N$ is waiting time between each excursion ($T^w_N = 0$ and $T^w_0$ is the time before the first excursion starts). These waiting times can all be bounded from above in expectation by $\frac{1}{p}$ since only the \OI operator can be stuck in this set of local maxima of $\jump_m$. Now, we can split the excursions into two families: the left-excursions (starting with the move from $n - m$ to $n - m - 1$) and the right-excursions (beginning with the move from $n - m$ to $n - m + 1$). Given an excursion $e_i$, where $1 \leq i \leq N$, then as all the left-excursions (resp. right-excursions) follow the same distribution, we know that if $e_i$ is a left-excursion, by \Cref{lem:drift-left-slope} we have
        \begin{align*}
             \esp[\lambda(e_i)] & \leq 1 + \O\left( m (1 + pn) \left(\frac{1}{p} + \frac{1}{q} \right) \right) \\ 
             & = \O\left( m (1 + pn) \left(\frac{1}{p} + \frac{1}{q} \right) \right),
        \end{align*}
    and if $e_i$ is a right-excursion, from the proof of \Cref{lem:avg-time-to-climb-jump} and \Cref{lem:drift-gap-region} we obtain
        \[ \esp[\lambda(e_i)] \leq 1 + \frac{n}{2} = O(n). \]

    Thus, as $O(n) = O(m n) = O\left( m (1 + pn) \left(\frac{1}{p} + \frac{1}{q} \right) \right)$ we conclude that
        \[ \esp[\lambda(e_i)] = \O\left( m (1 + pn) \left(\frac{1}{p} + \frac{1}{q} \right) \right). \]
    Hence, by Wald's theorem 
    we obtain
        \begin{align*}
            \esp[T_2] & = \esp[T_0^w] + \esp\left[ \sum_{i = 1}^N [ \lambda(e_i) + T^w_i] \right] \\ 
            & \leq \frac{1}{p} + \esp[N] \left( \O\left( m (1 + pn) \left(\frac{1}{p} + \frac{1}{q} \right) \right) + \frac{1}{p} \right) \\ 
            & = \O\left( \esp[N] m (1 + pn) \left(\frac{1}{p} + \frac{1}{q} \right) \right) \\ 
            & = \O\left(\frac{n^m}{(m - 1)! (1 - q)^{m - 1}} (1 + pn) \left(\frac{1}{p} + \frac{1}{q} \right) \right).
        \end{align*}
    Finally, building on what we computed, we have
    \begin{align*}
        \esp[T] & = \esp[T_1] + \esp[T_2] \\
        & = \O\left( (1 + pn) \left( \frac{1}{p} + \frac{1}{q} \right) N_{n, m, q} \right),
    \end{align*}
    where $N_{n, m, q} = n + \frac{n^m}{(m - 1)! (1 - q)^{m - 1}}$.
\end{proof}

\section{Runtime Analysis of the \MMAHH with OnlyWorsening Acceptance on \seqopt}\label{sec:runtime-analysis}

We now state and prove the runtime of our \MMAHH, which uses a Markov chain to select either the \OI or \OW acceptance operator, on the $\seqopt$ benchmark. 

\begin{theorem}[Runtime analysis of the \MMAHH]\label{thm:RuntimeMMAHH}
    Assume $p = \Theta(\frac{1}{n \log(n)} )$, $q = \Theta(\frac{1}{n \log(n)})$. Let $k \in [ 0 .. (n - 2) ]$ and $f \in \seqopt_k(d_1, \ldots, d_k)$ where $n > d_1 > \dots > d_k > 0 $. If $k = \O(1)$ then, the \MMAHH with \OI and \OW reaches the global maximum at $x^* = (1,\dots, 1)$ of $f$ in runtime $T$ with the expectation
        \[ \esp[T] = O \left( \frac{n^{k + 1}}{d_1 \cdots d_k} \log(n) \right). \]
\end{theorem}
\begin{proof}
    We will prove the desired result by induction on the local optima $d_\ell$ for $\ell\in[1..k]$. To do so we denote the first hitting time of a local optimum at $d_{\ell},$ by $T_{\ell}.$ The inductive hypothesis is then stated as follows,
    \begin{equation} \label{eqn:MMAHHInduction}
     \esp[T_{\ell}] = O \left(  \frac{n^{\ell}}{d_1 \cdots d_{\ell - 1}} \log(n) \right).  
    \end{equation}
    We first demonstrate \eqnref{eqn:MMAHHInduction} for the base case with $\ell=1.$ We denote by $k^*_1$ the number of phases needed to first reach a local optimum at $d_1$, i.e., $k^*_1 = \inf \left\{ k \in \N \mid y_{T^{(k)}_s} = d_1 \right\}$ and $T_1$ the time taken to reach said local optimum. Here we upper bound $\esp[T_1]$ by assuming $(y_0, s_0) = (n, \OW).$ Since for times $t\in[0..T_1]$ $x_t$ is confined to the slope, for which $y_t \in [d_1, n],$ with monotonic fitness value, our \emph{one-phase approximation} in \Cref{lem:one-phase-approx} applies and gives
        \[ \esp[k_1^*] \leq \frac{1}{p_n^0} = O(1), \]
    because $p, q = \Theta(\sfrac{1}{n \log(n)})$. Now, since 
        \[ T_1 \leq \sum\limits_{k=0}^{k^*_1 - 1} Z_k, \]
    and, by Lemma 22, 
    $\esp\left[ Z_k \right] \leq \max\left\{ \frac{1}{p}, \frac{1}{q} \right\} = O(n \log(n))$, we can use Wald's Theorem (see Theorem 17
    ) to obtain
        \[ \esp[T_1] \leq \frac{1}{p_n^0} \max\left\{ \frac{1}{p}, \frac{1}{q} \right\} = O(n \log(n)), \]
    as desired. This establishes \eqnref{eqn:MMAHHInduction} in the case $\ell = 1$.
        
     Now, for the induction step suppose that \eqnref{eqn:MMAHHInduction} is true for some $\ell \in[1..k]$. Without loss of generality, we assume that at $d_{\ell}$, the layer $\layer_{d_{\ell}}$ is a set of local maxima of $f$ (our argument equally holds for local minima by exchanging \OI and \OW). From these maxima, we define an excursion as a walk that starts by leaving $\layer_{d_{\ell}}$ and, either comes back to $\layer_{d_{\ell}}$ without having hit $\layer_{d_{\ell + 1}}$ (in case of a failure) or reach the new (unvisited) set of minima at layer $\layer_{d_{\ell + 1}}$ (in case of a success). We illustrate this situation in \Cref{fig:excursions-illustration}, where failing excursions are depicted in red and a successful excursion is shown in green.

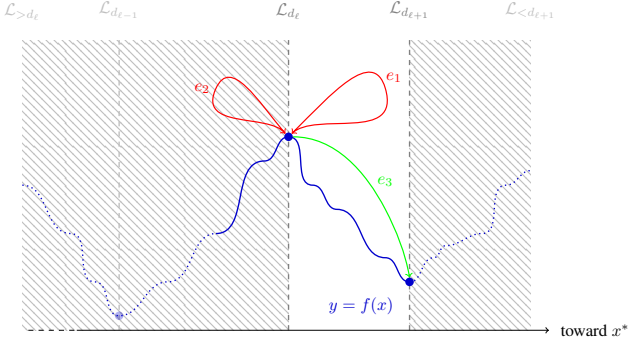
\begin{figure}[t]
    \centering
    \resizebox{\columnwidth}{!}{
    \begin{tikzpicture}   
        \node (D1) at (-6, 0) {}; \node (D2) at (-5, 0) {};
        \node (F) at (5, 0) {};

        \node (K) at (-0.5, -0.5) {}; \node (H) at (2, -0.5) {}; \node (L) at (-4, -0.5) {};

        \fill[pattern = north west lines, pattern color = lightgray] (2, 0) rectangle (4.5, 6) node[above = 0.3cm, lightgray] {$\layer_{< d_{\ell + 1}}$};
        \fill[pattern = north west lines, pattern color = lightgray] (-0.5, 0) rectangle (-6, 6) node[above = 0.3cm, lightgray] {$\layer_{> d_{\ell}}$};
        \draw[thick, dashed, gray] (K.center) ++(0, 0.5) -- ++(0, 6) node[above = 0.3cm] {$\layer_{d_{\ell}}$};
        \draw[thick, dashed, gray] (H.center) ++(0, 0.5) -- ++(0, 6) node[above = 0.3cm] {$\layer_{d_{\ell + 1}}$};
        \draw[thick, dashed, lightgray] (L.center) ++(0, 0.5) -- ++(0, 6) node[above = 0.3cm] {$\layer_{d_{\ell - 1}}$};
        

        \draw[dashed] (D1) -- (D2.east);
        \draw[->, solid] (D2) -- (F) node[right] {$\text{toward } x^*$};

        \draw[thick, blue!80!black] (-0.5, 4) to[out=0, in=180] (0, 3) to[out=0, in=180] 
             (0.5, 2.5) to[out=0, in=180] (1.5, 1.5) to[out=0, in=180] (2, 1);
        \draw[thick, blue!80!black, dotted] (2, 1) to[out=0, in=180] (2.5, 1.5) to[out=0, in=180] (2.75, 1.75) to[out=0, in=180] (3.5, 2) to[out=0, in=180] (4, 3) to[out=0, in=180] (4.5, 3.3);
        \draw[thick, blue!80!black, solid] (-2, 2) to[out=0, in=180]  (-1, 3.5) to[out=0, in=180] (-0.5, 4);
        \draw[thick, blue!80!black, dotted] (-6, 3) to[out=0, in=180] (-5, 2) to[out=0, in=180] (-4.5, 1) to[out=0, in=180] (-4, 0.3) to[out=0, in=180] (-3, 1) to[out=0, in=180] (-2, 2);

        \node[circle, fill = blue!80!black, inner sep = 0pt, minimum size = 5pt] (A) at (-0.5, 4) {};
        \node[circle, fill = blue!80!black, inner sep = 0pt, minimum size = 5pt] (B) at (2, 1) {};
        \node[circle, fill = blue!80!black, inner sep = 0pt, minimum size = 5pt, opacity = 0.3] (C) at (-4, 0.3) {};

				\node at (-1,4.5) (here) {};
				\node[blue!80!black] (f) at (1, 0.5) {$y = f(x)$};
				\node[red, left] (e2) at (-2, 5) {$e_2$}; \node[red, above right, xshift=-3pt] (e1) at (1.5, 5) {$e_1$}; \node[green, above, yshift=10pt] (e3) at (1.5, 2.5) {$e_3$};
				\draw[red, line width = 0.7pt,->] (A.north west) to[out angle=130, in angle=130, curve through = {(-2,5)}] (A.north west);
				\draw[red, line width = 0.7pt,->] (A.north east) to[out angle=45, in angle=45, curve through = {(1.5,5)}] (A.north east);
				\draw[green, line width = 0.7pt,->] (A.east) to[out angle = 0, in angle = 100, curve through = {(1.5, 2.5)}] (B.north);

    \end{tikzpicture}
    }
    \caption{\emph{Illustration of a function $f \colon \{0, 1\}^n \to \R$ with a set of local maxima in layer $\layer_{d_{\ell}}$ where three kinds of excursions can occur: a failing left excursion in red, the right excursion that fails to reach $\layer_{d_{\ell + 1}}$ also in red and the successful one in green.}}
    \label{fig:excursions-illustration}
\end{figure}
     
     Since we assume layer $\layer_{d_{\ell}}$ to be a set of local maxima of $f$, the average waiting time $T^w_i$ between the $(i - 1)$-th and $i$-th excursion is $\esp[T^w_i] = O(\frac{1}{p}) = O(n \log(n))$, where the constant does not depend on $i$. Further, let $k^*$ be the total number of excursions before layer $\layer_{d_{\ell + 1}}$ is reached for the first time and let $e_i$ be the $i$-th excursion of length $\lambda(e_i)$. Then, we have
        \[ T_{\ell+1} = T_{\ell} + \sum_{i = 1}^{k^*} \left(T^w_{i - 1} + \lambda(e_i) \right), \]
    where $T_0^w$ is the waiting time before the first excursion starts.
     
     Now, among the excursions, there are those starting by accepting the flip of a one-bit (which we call the left-excursions) and the others, starting by accepting a zero-bit flip (the right-excursions). First, the length of any left-excursion can be upper-bounded by $\esp\left[ T_{\ell} \right]$. Second, as any right-excursion is confined in the slope $[d_{\ell + 1}, d_{\ell}]$ where $f$ is decreasing across the layers $\layer_{d_{\ell}}$, $\ldots$, $\layer_{d_{\ell + 1}}$, our \emph{one-phase approximation} in \Cref{lem:one-phase-approx} applies once more and gives, here for \OW, that an average number of $O(1)$ right-excursions are needed in order to reach layer $\layer_{d_{\ell + 1}}$ and the total length of all these right-excursions is $O(n \log(n))$.

    Moreover, upon leaving $\layer_{d_{\ell}}$, there is a probability $\frac{d_{\ell}}{n}$ (resp. $\frac{n - d_{\ell}}{n}$) that the excursion will be a right-excursion (resp. left-excursion) hence, the average number of left-excursions performed before a right-excursion occurs is $\frac{n}{d_{\ell}}$ thus the average number of excursions needed is given by
        \[ \esp[k^*] = O\left( \frac{n}{d_{\ell}} \right), \]
    using Wald's theorem~\cite{Wald44} in the simplified version of~\cite{DoerrK15}.
    
    Finally, with Wald's theorem again and using the induction hypothesis on $\esp[T_{\ell}]$, we obtain
    \begin{align*}
        \esp[T_{\ell + 1}] & \leq \esp[T_\ell] \\
        & \ \ + \esp[k^*] \left( O(n \log(n)) + O\left( \frac{n^{\ell}}{d_1 \cdots d_{\ell - 1}} \log(n)  \right) \right) \\ 
        & = \O\left( \frac{n^{\ell + 1}}{d_1 \cdots d_{\ell}} \log(n) \right),
    \end{align*}
    since, for any $ i \in[1..k^*]$, we have
        \[ \esp[ T_{i - 1}^w + \lambda(e_i)] = \esp[T_{i - 1}^w] + \esp[\lambda(e_i)] \leq \O(n \log(n)) + \esp[T_{\ell}], \]
    where the left-hand side does not depend on $i$. The stated result hence follows by considering the case $\ell = k + 1$ in~\eqnref{eqn:MMAHHInduction}, where assumption $k = \O(1)$ implies that that the hidden constant in~\eqnref{eqn:MMAHHInduction}, which is a $\O(1)^k$, is still a $\O(1)$.

    

\end{proof}

The factor $O(n \log(n))$ in the complexity derived in {\normalfont \Cref{thm:RuntimeMMAHH}} represents the expected time the \MMAHH requires to transition from one local optimum to a neighboring local optimum. The other multiplicative factors, i.e., $\O( \sfrac{n^k}{(d_1 \cdots d_k)})$, can be interpreted as the time needed for a biased random walk on the set of local optima, starting from  $(0, \dots, 0)$, to reach the global maximum.

The generality of our \seqopt benchmark allows us to immediately extend our result in Theorem \ref{thm:RuntimeMMAHH} to a variety of known benchmark functions, for which the \MMAHH exhibits remarkable performance. 


\begin{corollary}\label{runtimebenchmarks-jump}
    Let $p = \Theta(\sfrac{1}{n \log(n)})$ and $q = \Theta(\sfrac{1}{n \log(n)})$. Then, the runtime $T$ of the \MMAHH
\begin{enumerate}
    \item on $\jump_m$ with $m\in [2..\frac{n}{2}]$ satisfies
    $$\esp(T) = O\left(\frac{n^3}{m} \log(n)\right),$$
    \item on $\cliff_d$ with $d\in[2..\frac{n}{2}]$ satisfies
    $$\esp(T)  = \O\left(\frac{n^3}{d^2} \log(n)\right),$$
    \item on $\cliffjump_{d, r, s}$ with $s\in \mathbb{N}^+,$ $d,r\in[1..\frac{n}{2}]$ with $r<d$ satisfies
    $$\esp(T) = O\left(\frac{n^3}{d (d-r)} \log(n)\right).$$
\end{enumerate}
\end{corollary}
These results highlight the power of \MMAHH: on $\cliff_d$ we incur only an extra $O(\log n)$ factor compared to the \MAHH\!\!'s $O(n\log n + \frac{n^3}{d^2})$ runtime (using probability $p=1 / ((1 + \varepsilon)n)$) as stated in \cite{LissovoiOW23}, while achieving drastic speed‐ups on $\jump_m$, reducing the $O\bigl(n \log n + \frac{n^{2m}(1+\varepsilon)^{m-1}}{(m^2\,m!)}\bigr)$ bound of the \MAHH, and similarly outperforming it on $\cliffjump_{d,r,s}$ beyond the previously established $O\bigl((d-r)\frac{(1+\varepsilon)^r n^{2r+2}}{(r+1)^2(r+1)!}\bigr)$ result.

\section{Conclusion}

In this work, we proposed two new building-blocks for the design of efficient move-acceptance hyper-heuristics. The Markov chain-based selection of the acceptance operator instead of the random mixing strategy used in previous theoretical works already with non-elaborate switching probabilities like $q = \frac 12$ greatly increases the rate of longer phases of the non-elitist selection operator (\AM or \OW), which led to provable significant speed-ups. The new OnlyWorsening (\OW) acceptance operator was designed to aid the heuristic leave local optima. Used with the classic \MAHH it gives comparable results as \AM but, together with the Markov selection strategy it can lead to drastic speed-ups, e.g., polynomial runtimes on all $\seqopt_k$ functions, $k$ a constant. Both from the runtime guarantees proven in this work and from the working principles visible in our proofs we are very optimistic that our new building-blocks have a true potential for improving the hyper-heuristics used today.  

The presented \MMAHH algorithm mixing between \OI and \OW cannot traverse plateaus in the fitness landscape. The extension of these acceptance operators to accept solutions of equal fitness is an interesting direction for follow-up work. Furthermore, our theoretical analysis suggests that the empirical exploration of our proposed \MMAHH is a promising direction for future research.

\subsubsection{Acknowledgements}

The research of B. Doerr is supported by the FMJH Program PGMO. 
J. Lutzeyer is supported by the French National Research Agency (ANR) via the ``GraspGNNs'' JCJC grant (ANR-24-CE23-3888).

}

\bibliographystyle{named}
\bibliography{ich_master,alles_ea_master,rest}

\cleardoublepage
\appendix
\onecolumn


\section{Mathematical Tools}\label{subsec:math-tools}
The following theorem is used multiple times through the runtime analysis of the \MMAHH to upper bound the expectation of a sum of random variables whose range also depends on a random variable.

\begin{theorem}[Simplified version of Wald’s equation -- \cite{Wald44}; \cite{DoerrK15}]\label{thm:wald}
    Let $T$ be a random variable with bounded expectation and let $(X_t)_{t \geq 0}$ be a sequence of non-negative random variables. If there exists a constant $C > 0$ such that for any $i \geq 0$ we have $\esp[ X_i \mid T \geq i] \leq C$ then
        \[ \esp\left[ \sum_{i = 1}^T X_i \right] \leq C \esp[T]. \]
\end{theorem}

We now state two drifts theorems, one is a more general version of the standard additive drift theorem and allows \emph{overshooting}, i.e.,  the sequence of random variables can exceed the target but at the price of an additional terms, the expected overshooting, in the upper bound. The other one is the multiplicative drift theorem which provides an upper bound when the expected progress linearly depends on the distance to the target.


\begin{theorem}[Additive drift theorem with overshooting -- \cite{KotzingK19}]\label{thm:add-drift-overshooting}
    Let $\alpha \leq 0$ be a non-positive real number, $(X_t)_{t \geq 0}$ a sequence of random variables over $\R$ and $T = \inf\{ t \geq 0 \mid X_t \leq 0\}$. Suppose that for any $0 \leq t \leq T$, $X_t \geq \alpha$ and there exists $\delta > 0$ such that for all $0 \leq t < T$, 
        \[ X_t - \esp[X_{t + 1} \mid X_0, \ldots, X_t] \geq \delta, \]
    then
        \[ \esp[T \mid X_0] \leq \frac{X_0 - \esp[X_T \mid X_0]}{\delta} \leq \frac{X_0 - \alpha}{\delta}. \]
\end{theorem}

\begin{theorem}[Multiplicative drift theorem -- \cite{DoerrJW12algo}]\label{thm:mult-drift}
    Let $S \subseteq \R_+^*$ be a finite state space of positive real numbers with minimum $s_{\min}$ and $(X_t)_{t\geq0}$ a sequence of random variables over $S \cup \{0\}$. Let $T$ be the first time $t \in \N$ for which $X_t = 0$ and suppose further that there exists a constant $\delta > 0$ such that 
    \[
    \esp\left[X_{t} - X_{t+1} \mid  X_t = s \right] \geq \delta s,
    \]
    holds for all $s \in S$ with $\Pr\left[X_t = s\right] > 0$. Then, for all $s_0 \in S$ with $\Pr\left[X_0 = s_0 \right] > 0$,
    \[
    \esp\left[T \mid  X_0 = s_0 \right] \leq \frac{1 + \log(s_0/s_{\min})}{\delta}.
    \]
\end{theorem}

\section{Tools for the Analysis of the \MMAHH}\label{app:tools-analysis-mmahh}

\begin{proof}[Proof of \Cref{lem:minimality-p^0_n}]
    We prove that for a fixed $0 \leq h \leq n$ then the probabilities $(\Pr[y_{T_s^{(1)}} \leq h \mid y_0 = k, s_0 = \OI])_{0 \leq k  \leq n}$ are non-increasing, i.e.,
        \begin{align*} 
            \Pr[y_{T_s^{(1)}} \leq h \mid y_0 = n, s_0 = \OI] & \leq \Pr[y_{T_s^{(1)}} \leq h \mid y_0 = n - 1, s_0 = \OI] \\ 
            & \leq \ldots \leq \Pr[y_{T_s^{(1)}} \leq h \mid y_0 = 0, s_0 = \OI] = 1, 
        \end{align*}
    and for a fixed $0 \leq k \leq h$, the probabilities $(\Pr[y_{T_s^{(1)}} \leq h \mid y_0 = k, s_0 = \OI])_{0 \leq h  \leq n}$ are non-decreasing, i.e.,
        \begin{align*} 
            \Pr[y_{T_s^{(1)}} \leq 0 \mid y_0 = k, s_0 = \OI] & \leq \Pr[y_{T_s^{(1)}} \leq 1 \mid y_0 = k, s_0 = \OI] \\ 
            & \leq \ldots \leq \Pr[y_{T_s^{(1)}} \leq n \mid y_0 = k, s_0 = \OI] = 1. 
        \end{align*}
    Note that these two results imply the minimality of $p_n^0$ that we claim in our lemma.
    
    For the first part, fix $h$ and let $0 \leq k \leq n$ then, when $k < h$ one has:
    \[ \Pr[y_{T_s^{(1)}} \leq h \mid y_0 = k + 1, s_0 = \OI] = 1 \leq 1 = \Pr[y_{T_s^{(1)}} \leq h \mid y_0 = k, s_0 = \OI] \]
    and when $h \leq k < n$, we have:
        \begin{align*}
            \Pr[y_{T_s^{(1)}} \leq h \mid y_0 & = k + 1, s_0 = \OI] \\ 
            & = \frac{1}{1 - p} \frac{\Gamma(k + 2) \Gamma\left( \frac{n p}{1 - p} + h + 1 \right)}{\Gamma(h + 1) \Gamma\left( \frac{n p}{1 - p} + k + 2 \right)} \\ 
            & = \left( \frac{k + 1}{\frac{n p}{1 - p} + k + 1} \right) \frac{1}{1 - p} \frac{\Gamma(k + 1) \Gamma\left( \frac{n p}{1 - p} + h + 1 \right)}{\Gamma(h + 1) \Gamma\left( \frac{n p}{1 - p} + k + 1 \right)} \\ 
            & = \left( \frac{k + 1}{\frac{n p}{1 - p} + k + 1} \right) \Pr[y_{T_s^{(1)}} \leq h \mid y_0 = k, s_0 = \OI] \\ 
            & < \Pr[y_{T_s^{(1)}} \leq h \mid y_0 = k, s_0 = \OI],
        \end{align*}
    because $\frac{n p}{1 - p} > 0$ since $p > 0$.

    Now, for the other part, fix $k$ and let $0 \leq h \leq n$ then, when $k < h$ one has:
    \[ \Pr[y_{T_s^{(1)}} \leq h - 1 \mid y_0 = k, s_0 = \OI] = 1 \leq 1 = \Pr[y_{T_s^{(1)}} \leq h \mid y_0 = k, s_0 = \OI] \]
    and when $0 < h \leq k$, we have:
        \begin{align*}
            \Pr[y_{T_s^{(1)}} \leq h - 1 \mid y_0 & = k, s_0 = \OI] \\ 
            & = \frac{1}{1 - p} \frac{\Gamma(k + 1) \Gamma\left( \frac{n p}{1 - p} + h \right)}{\Gamma(h) \Gamma\left( \frac{n p}{1 - p} + k + 1 \right)} \\ 
            & = \left( \frac{h}{\frac{n p}{1 - p} + h} \right) \frac{1}{1 - p} \frac{\Gamma(k + 1) \Gamma\left( \frac{n p}{1 - p} + h + 1 \right)}{\Gamma(h + 1) \Gamma\left( \frac{n p}{1 - p} + k + 1 \right)} \\ 
            & = \left( \frac{h}{\frac{n p}{1 - p} + h} \right) \Pr[y_{T_s^{(1)}} \leq h \mid y_0 = k, s_0 = \OI] \\ 
            & < \Pr[y_{T_s^{(1)}} \leq h \mid y_0 = k, s_0 = \OI],
        \end{align*}
    because $\frac{n p}{1 - p} > 0$ since $p > 0$. This is what we wanted to prove.
\end{proof}

We note the following elementary results following directly from the definition of the \MMAHH.

\begin{lemma}\label{lem:markov-chain}
   The sequence $\left( (x_t, s_t) \right)_{t \in \N}$ is a homogeneous Markov chain. Moreover, when $f \in \seqopt_k$ with $k \in \N$ then $\left( (y_t, s_t) \right)_{t \in \N}$ is also a homogeneous Markov chain.
\end{lemma}

\begin{proof}[Proof of \Cref{lem:markov-chain}]
    We have 
    \begin{equation}
        (x_{t+1}, s_{t+1}) = g((x_t, s_t), (z_t, v_t)),
    \end{equation}
    where $z_t \overset{\rm i.i.d.}{\sim} \mathcal U([0, 1])$ is a random variable sampled uniformly in $[0, 1]$ independently from $(x_t, s_t)$, $v_t \overset{\rm i.i.d.}{\sim} \mathcal U([1..n])$ is a random variable sampled uniformly in $[1 ..n]$ independently from $(x_t, s_t)$ and $z_t$, and the function $g$ is defined as follows
    \begin{equation}
        g((x, s), (z, v)) = \begin{cases}
            (x_v', h(s, z)), & \text{ if } s = \OI \text{ and } f(x_v') > f(x);  \\ 
            (x, h(s, z)), & \text{ if } s = \OI \text{ and } f(x_v') \leq f(x); \\
             (x_v', h(s, z)), & \text{ if } s = \OW \text{ and } f(x_v') < f(x); \\
            (x, h(s, z)), & \text{ if } s = \OW \text{ and } f(x_v') \geq f(x);
        \end{cases}
    \end{equation}
    where $x_v'$ is obtained from $x$ by flipping its $v$-th bit,  and we denote by $h$ the function defined by
    \begin{equation}
        h(s, z) = \begin{cases}
            \OI, & \text{ if } s = \OI \text{ and } z \leq 1 - p; \\
            \OW, & \text{ if } s = \OI \text{ and } z > 1 - p; \\
            \OW, & \text{ if } s = \OW \text{ and } z \leq 1 - q; \\
            \OI, & \text{ if } s = \OW \text{ and } z > 1 - p.
        \end{cases} \label{eq:def-function-h}
    \end{equation}
    This proves that $\left( (x_t, s_t) \right)_{t \in \N}$ is a homogeneous Markov chain.
    
    Furthermore, by definition $y_t = \Hamming(x_t, x^*)$, hence when $f \in \seqopt_k$ we have $y_t = n - \| x_t \|_1$. The following holds
    \begin{equation}
        (y_{t+1}, s_{t+1}) = \ell((y_t, s_t), (z_t', v_t')),
    \end{equation}
     where $z_t' \overset{\rm i.i.d.}{\sim} \mathcal U([0, 1])$ is a random variable sampled uniformly in $[0, 1]$ independently from $(y_t, s_t)$, $v_t' \overset{\rm i.i.d.}{\sim} \mathcal U([0, 1])$ is a random variable sampled uniformly in $[0, 1]$ independently from $(y_t, s_t)$ and $z_t'$, and the function $\ell$ is defined as follows
     \begin{equation}
         \ell((y, s), (z, v)) = \begin{cases}
             (y, h(s, z)), & \text{ if } s = \OI, v < \frac{y}{n},   \layer_{y - 1} \overset{f}{\prec} \layer_{y};\\
             (y, h(s, z)), & \text{ if } s = \OI, v \geq \frac{y}{n},  \layer_{y+1} \overset{f}{\prec} \layer_{y}; \\
             (y, h(s, z)), & \text{ if } s = \OW, v < \frac{y}{n},  \layer_{y} \overset{f}{\prec} \layer_{y-1}; \\
             (y, h(s, z)), & \text{ if } s = \OW, v \geq \frac{y}{n},  \layer_{y} \overset{f}{\prec} \layer_{y+1}; \\
             (y - 1, h(s, z)), & \text{ if } s = \OI, v < \frac{y}{n},  \layer_{y} \overset{f}{\prec} \layer_{y-1};\\
             (y - 1, h(s, z)), & \text{ if } s = \OW, v < \frac{y}{n},  \layer_{y-1} \overset{f}{\prec} \layer_{y};\\
             (y + 1, h(s, z)), & \text{ if } s = \OI, v \geq \frac{y}{n},  \layer_{y} \overset{f}{\prec} \layer_{y+1}; \\
             (y + 1, h(s, z)), & \text{ if } s = \OW, v \geq \frac{y}{n},  \layer_{y+1} \overset{f}{\prec} \layer_{y};
         \end{cases}
     \end{equation}
     where the function $h$ is defined in the same manner as in \eqref{eq:def-function-h}, and the random variable $v_t$ is used to express wether the flip at time $t$ is a $0$-flip or a $1$-flip. This concludes the proof.
\end{proof}

\begin{lemma} \label{lem:independance}
    For any $k \in \N$, the random variables $Z_k$ and $x_{T_s^{(k)}}$ are independent. Consequently, the conditional law of $Z_k$ under $s_0 = s \in \{\OI, \OW\}$ does not depend on $x_{T_s^{(k)}}$ -- and thus also not on $y_{T_s^ {(k)}}$.
\end{lemma}

\begin{proof}[Proof of \Cref{lem:independance}]
    Fix $\ell \in \N^*$, we have by definition $Z_k = T_s^{(k+1)} - T_s^{(k)}$ where $T_s^{(k+1)} = \inf\{t \geq T_s{(k)}: s_{t} \neq s_{T_s^{(k)}} \}$. Hence $Z_k$ is fully determined by the $\sigma$-algebra $\mathcal F$ generated by the random variables $\{s_t: t \geq 0 \}$. Since $\{s_t\}_{t\geq 0}$ is a Markov chain ($s_{t+1}$ depends only on $s_t$), we have 
    \begin{equation}
        \forall \mathcal A \in \mathcal F,\;  \Pr[\mathcal A \mid x_{T_s^{(k)}}] = \Pr[\mathcal A],
    \end{equation}
    thus we have $\Pr[Z_k \mid x_{T_s^{(k)}}] = \Pr[Z_k]$. This yields the independence of $Z_k$ and $x_{T_s^{(k)}}$.
\end{proof}

\begin{lemma}\label{lem:geometric-law}
    For any $k \in \N$ we have
        \[ Z_k \sim U_k, \]
    where $U_k$ is a random geometric variable of parameter $p$ when $k$ is even and $q$ otherwise.
\end{lemma}
\begin{proof}[Proof of \Cref{lem:geometric-law}]
    If $k$ is even, since $s_0 = \OI$ we have $s_{T_s^{(k)}} = \OI$ and $s_{T_s^{(k+1)}} = \OW$. Fix $\ell \in \N^*$, we have 
    \begin{align}
        \Pr[Z_k = \ell] &= \Pr[T_s^{(k+1)} - T_s^{(k)} = \ell]\\
        &= \sum_{j \geq 0} \Pr[T_s^{(k+1)} - T_s^{(k)} = \ell \mid T_s^{(k)} = j] \Pr[T_s^{(k)} = j] \\
        &= \sum_{j \geq 0} \Pr[T_s^{(k+1)} = j + \ell \mid T_s^{(k)} = j] \Pr[T_s^{(k)} = j] \\
        & = \sum_{j \geq 0} \Pr[s_{j+ \ell} = \OW, s_{j+ \ell - 1} = \OI, \dots, s_{j+1} = \OI \mid T_s^{(k)} = j] \Pr[T_s^{(k)} = j]  \\
        &= \sum_{j \geq 0} (1- p)^{\ell - 1} p \Pr[T_s^{(k)} = j]  \\
        &= (1- p)^{\ell - 1} p \sum_{j \geq 0} \Pr[T_s^{(k)} = j] \\
        &= (1 - p)^{\ell - 1} p,
    \end{align}
    hence $Z_k$ follows the geometric law of parameter $p$ as claimed in the lemma.
    
    Similarly, if $k$ is odd we have $s_{T_s^{(k)}} = \OW$ and $s_{T_s^{(k+1)}} = \OI$. Fix $\ell \in \N^*$, we have 
    \begin{align}
        \Pr[Z_k = \ell] &= \Pr[T_s^{(k+1)} - T_s^{(k)} = \ell]\\
        &= \sum_{j \geq 0} \Pr[T_s^{(k+1)} - T_s^{(k)} = \ell \mid T_s^{(k)} = j] \Pr[T_s^{(k)} = j] \\
        &= \sum_{j \geq 0} \Pr[T_s^{(k+1)} = j + \ell \mid T_s^{(k)} = j] \Pr[T_s^{(k)} = j] \\
        & = \sum_{j \geq 0} \Pr[s_{j+ \ell} = \OI, s_{j+ \ell - 1} = \OW, \dots, s_{j+1} = \OW \mid T_s^{(k)} = j] \Pr[T_s^{(k)} = j]  \\
        &= \sum_{j \geq 0} (1- q)^{\ell - 1} q \Pr[T_s^{(k)} = j]  \\
        &= (1- q)^{\ell - 1} q \sum_{j \geq 0} \Pr[T_s^{(k)} = j] \\
        &= (1 - q)^{\ell - 1} q,
    \end{align}
    hence $Z_k$ follows the geometric law of parameter $q$ as claimed.
\end{proof}

The heart of our analysis is the following proof of \Cref{lem:oi-prob}.

\begin{proof}[Proof of \Cref{lem:oi-prob}]
    Fix $\ell \in \N^*$, we denote by $p_{k, \ell}^h$ the following probability
        \[ p_{k, \ell}^h = \Pr[y_{\ell} \leq h \mid y_0 = k, T_s^{(1)} = \ell, s_0 = \OI], \]
    which is the probability that at some time $0 \leq t \leq \ell$ we reach a bit-string $x'$ such that $\Hamming(x', x^*) \leq h$ after performing consecutively $\ell$ times the \OI acceptance operator. 
    Now, given $0 \leq k \leq n$ and $\ell \geq 1$, we have the boundary conditions $p_{k, \ell}^h = 1$ if $k \leq h$ and $p_{k, \ell}^h = 0$ if $h < k$ and $\ell < k - h$. Otherwise, when $h < k$, $\ell \geq k - h$ and $\ell \geq 2$, the following recurrence
        \begin{equation}\label{oi-prob:eq1}
            p_{k, \ell}^h = \left( \frac{n - k}{n} \right) p_{k, \ell - 1}^h + \left( \frac{k}{n} \right) p_{k - 1, \ell - 1}^h, 
        \end{equation} 
    holds and is derived using a first step analysis of the Markov chain $((x_t, s_t))_{t \in \N}$ and as $\ell \geq 2$ then at least $s_0 = \OI = s_1$. Notice that relation \eqnref{oi-prob:eq1} still holds in the cases $k \leq h$ (since also $k - 1 \leq h)$ but also when both $h < k$ and $\ell < k - h$ since then $\ell - 1 < k - h$ and $\ell - 1 < (k - 1) - h$.

    That being said, we can now split the probability $p_k^h$ where $0 \leq k \leq n$ is such that\footnote{Otherwise, if $0 \leq k \leq h$ then $p_k^h = 1$ as we will see.} $h < k$ as
        \begin{align*}
            p_k^h & = \Pr[y_{T_s^{(1)}} \leq h \mid y_0 = k, s_0 = \OI] \\ 
            & = \sum_{\ell = 1}^{\infty} \Pr[y_{\ell} \leq h \mid y_0 = k, T_s^{(1)} = \ell, s_0 = \OI] \Pr[T_s^{(1)} = \ell \mid y_0 = k, s_0 = \OI] \\
            & = \sum_{\ell = 1}^{\infty} p_{k, \ell}^h \Pr[Z_0 = \ell \mid s_0 = \OI] \numberthis\label{eq:prob-1} \\ 
            & = \sum_{\ell = 1}^{\infty} p_{k, \ell}^h (1 - p)^{\ell - 1} p \\
            & = p_{k, 1}^h p + \sum_{\ell = 2}^{\infty} \left( \left( \frac{n - k}{n} \right) p_{k, \ell - 1}^h + \left( \frac{k}{n} \right) p_{k - 1, \ell - 1}^h \right) (1 - p)^{\ell - 1} p \\ 
            & = p_{k, 1}^h p + \left( \frac{n - k}{n} \right) (1 - p) \sum_{\ell = 2}^{\infty} p_{k, \ell - 1}^h (1 - p)^{\ell - 2} p \\
            &\hphantom{= p_{k, 1}^h p\ }+ \left( \frac{k}{n} \right) (1 - p) \sum_{\ell = 2}^{\infty} p_{k - 1, \ell - 1}^h (1 - p)^{\ell - 2} p \\ 
            & = p_{k, 1}^h p + \left( \frac{n - k}{n} \right) (1 - p) p_k^h + \left( \frac{k}{n} \right) (1 - p) p_{k - 1}^h \numberthis\label{eq:prob-2},
        \end{align*}
    where, in \eqnref{eq:prob-1} we use \Cref{lem:independance} since $T_s^{(1)} = Z_0$ denotes the length of the first phase of \OI (as $s_0 = \OI$).

    Then to solve this recurrence, first, one need to find an expression for $p_{k, 1}^h$, which is
        \[ p_{k, 1}^h = \begin{cases} 0, & \text{if $h + 1 < k$;} \\ \frac{h + 1}{n}, & \text{if $k = h + 1$;} \\ 1, & \text{if $k \leq h$;} \end{cases} \]
    and
        \[ p_h^h = \sum_{\ell = 1}^{\infty} p_{h, \ell}^h (1 - p)^{\ell - 1} p = \sum_{\ell = 1}^{\infty} (1 - p)^{\ell - 1} p = 1. \]

    Next, we distinguish the cases $k = h + 1$ and $k > h + 1$. First of all, if $k = h + 1$ then \eqnref{eq:prob-2} becomes
        \begin{align*} p_{h + 1}^h & = p_{h + 1, 1}^h p + \left( \frac{n - (h + 1)}{n} \right) (1 - p) p_{h + 1}^h + \left( \frac{h + 1}{n} \right) (1 - p) p_h^h \\ 
        & = \frac{h + 1}{n} +  \left( \frac{n - (h + 1)}{n} \right) (1 - p) p_{h + 1}^h,
        \end{align*}
    hence,
        \begin{align*} 
            p_{h + 1}^h & = \frac{h + 1}{n} \frac{1}{1 - \left( \frac{n - h - 1}{n} \right) (1 - p)} \\ 
            & = \frac{h + 1}{n - (n - h - 1)(1 - p)} \\
            & = \frac{1}{1 - p} \frac{h + 1}{\frac{n}{1 - p} - n + h + 1} \\
            & = \frac{1}{1 - p} \frac{h + 1}{\frac {n p}{1 - p} + h + 1}.
        \end{align*}

    On the other hand, when $k > h + 1$ then $p_{k, 1}^h = 0$ and grouping together the two terms in $p_k^h$, \eqnref{eq:prob-2} now becomes
        \begin{align*}
            p_k^h & = \frac{1}{1 - \left( \frac{n - k}{n} \right) (1 - p)} \left( \frac{k}{n} \right) (1 - p) p_{k - 1}^h \\ 
            & = \frac{k}{n - (n - k)(1 - p)} (1-p) p_{k - 1}^h \\ 
            & = \frac{k}{\frac{n}{1 - p} - n + k} p_{k - 1}^h \\ 
            & = \frac{k}{\frac{n p}{1 - p} + k} p_{k - 1}^h.
        \end{align*}

    Combining both cases, this leads to the following formula for $p_k^h$ when $h < k \leq n$
        \begin{align*} 
            p_k^h & = \frac{1}{1 - p} \prod_{j = h + 1}^k \left( \frac{j}{\frac{n p}{1 - p} + j} \right) \\
            & = \frac{1}{1 - p} \frac{k!}{h!} \frac{1}{\left( \frac{n p}{1 - p} + k \right) \cdots \left( \frac{n p}{1 - p} + h + 1 \right)} \\ 
            & = \frac{1}{1 - p} \frac{\Gamma(k + 1) \Gamma\left( \frac{n p}{1 - p} + h + 1 \right)}{\Gamma(h + 1) \Gamma\left( \frac{n p}{1 - p} + k + 1 \right)},
        \end{align*}
    where $\Gamma$ is the Gamma function. Moreover, when $k \leq h$ then $p_k^h = 1$ which is intuitive since the \OI operator cannot climb down \onemax. This proves the \Cref{lem:oi-prob} as desired.
\end{proof}

\begin{proof}[Proof of \Cref{lem:limit}]
    By \Cref{lem:oi-prob}, for any integer $n > 0$,
        \[ p_n^0 = \frac{1}{1 - p} \frac{\Gamma(n + 1) \Gamma\left( \frac{n p}{1 - p} + 1 \right)}{\Gamma\left( \frac{n p}{1 - p} + n + 1 \right)}, \]
    and, as $p = \frac{1}{c n \log(n)} = \o_{n \to +\infty}\left( \frac{1}{n} \right)$ then $\frac{1}{1 - p} = 1 + o(1)$ and $\frac{n p}{1 - p} = o(1)$. Moreover, by the continuity of the Gamma function over $\R_+^*$ 
        \[ \Gamma\left( \frac{n p}{1 - p} + 1 \right) = \Gamma(1) + o(1) = 1 + o(1). \]
    
    To derive the claimed asymptotic, we will use the Stirling's approximation for the gamma function which can be found in \cite[5.11.7]{NIST:DLMF}. We have
    \begin{align*}
        p_n^0 & = \frac{1}{1 - p} \frac{\Gamma(n + 1) \Gamma\left( \frac{n p}{1 - p} + 1 \right)}{\Gamma\left( \frac{n p}{1 - p} + n + 1 \right)} \\ 
        & = \frac{n!}{\Gamma\left( \frac{n}{1 - p} + 1 \right)} (1 + o(1)) \\
        & = \frac{\left( \frac{n}{e} \right)^n \sqrt{2 \pi n}}{\left( \frac{n}{e (1 - p)} \right)^{\frac{n}{1 - p}} \sqrt{\frac{2 \pi n}{1 - p}}} (1 + o(1)),
    \end{align*}
    and the factor in front of the $(1 + o(1))$ becomes
        \begin{align*}
            \frac{\left( \frac{n}{e} \right)^n \sqrt{2 \pi n}}{\left( \frac{n}{e (1 - p)} \right)^{\frac{n}{1 - p}} \sqrt{\frac{2 \pi n}{1 - p}}} & = \left( \frac{n}{e} \right)^n \left( \frac{e (1 - p)}{n} \right)^{\frac{n}{1 - p}} \sqrt{1 - p} \\ 
            & = \exp\left( n \log(n) - n + \frac{n}{1 - p} \left( \log(1 - p) + 1 - \log(n) \right) \right) \sqrt{1 - p},
        \end{align*}
    \hypertarget{exp-factor}{}
    where $\sqrt{1 - p} = 1 + o(1)$ while, for the exponential factor
    \begin{align*}
        &\exp\left( n \log(n) - n + \frac{n}{1 - p} \left( \log(1 - p) + 1 - \log(n) \right) \right) \\ 
        & = \exp\left( n \log(n) - n + \frac{n}{1 - p} - n \log(n) \left( 1 + p + O(p^2) \right) + \frac{n \log(1 - p)}{1 - p} \right) \\ 
        & = \exp\left( \frac{n p}{1 - p} - n \log(n) p + O\left( \frac{1}{n \log(n)} \right) + \frac{n}{1 - p} \left( -p + o(p) \right) \right) \\ 
        & = \exp\left( -\frac{1}{c} + o(1) \right). 
    \end{align*}

    Thus, we conclude that
        \[ \Pr[y_{T_s^{(1)}} = 0 \mid y_0 = n, s_0 = \OI] = \left( 1 + \o_{n \to +\infty}(1) \right)e^{-1 / c}, \]
    as desired.
\end{proof}

\section{Drift Lemmas}

\begin{proof}[Proof of \Cref{lem:driftAM}]
    First, as the length of a phase is also random, we write for $\ell \in \N_{>0}$
    \[ \Delta^{\AM}_{i, \ell} = \esp\left[ y_0 - y_{T_s^{(1)}} \mid \, y_0 = i, Z_0 = \ell, s_0 = \AM \right], \]
    the drift after performing exactly $\ell$ times the \AM operator in a row when starting in $y_0 = i$. Then, by conditioning on the length of the phase \AM, we obtain
        \begin{align*}
            \Delta^{\AM}_i & = \sum_{\ell = 1}^{\infty}  \Delta^{\AM}_{i, \ell} \Pr\left[ Z_{2 k} = \ell \mid y_{T_s^{(2k)}} = i, s_{T_s^{(2k)}} = \AM\right] \\
            & = \sum_{\ell = 1}^{\infty}  \Delta^{\AM}_{i, \ell} \Pr\left[ Z_{2 k} = \ell \mid s_{T_s^{(2k)}} = \AM \right] \numberthis\label{lem:driftAM-eq1} \\
            & = \sum_{\ell = 1}^{\infty}  \Delta^{\AM}_{i, \ell} (1 - q)^{\ell - 1} q, \numberthis\label{conditioningAM}
        \end{align*}
    where in \eqnref{lem:driftAM-eq1} we use \Cref{lem:independance} stating that variables $Z_{2 k}$ and $y_{T_s^{(2k)}}$ are independent. Then, for $\ell > 0$ and $j \in [0.. \ell - 1]$ let
    \[\Delta^{\AM}_{i, j, \ell} = \esp\left[ y_{T_s^{(2k)} + j} - y_{T_s^{(2k)} + j + 1} \mid \ y_{T_s^{(2 k)}} = i, Z_{2 k} = \ell, s_{T_s^{(2 k)}} = \AM \right], \]
    hence
    \[\Delta^{\AM}_{i, \ell} = \sum_{j = 0}^{\ell - 1} \Delta^{\AM}_{i, j, \ell}, \]
    and we now show by finite induction on $j \in [0.. \ell - 1]$ that
    \begin{equation}\label{lem:driftAM-HR}
        \forall i \in [0.. n],\  \Delta^{\AM}_{i, j, \ell} = \left(\frac{n-2}{n}\right)^j \frac{2i-n}{n}.    
    \end{equation}

    First, for $j = 0$, we have
    \begin{align*}
        \Delta^{\AM}_{i, 0, \ell} &= \esp\left[ y_{T_s^{(2k)}} - y_{T_s^{(2k)} + 1} \mid \ y_{T_s^{(2 k)}} = i, Z_{2 k} = \ell, s_{T_s^{(2 k)}} = \AM \right]\\ 
        &= \frac{i}{n} (1) + \frac{n-i}{n} (-1)\\
        &= \frac{2i - n}{n}.
    \end{align*}

    Then, if we suppose the result \eqnref{lem:driftAM-HR} correct for some $j \in [0..\ell - 2]$, we obtain
    \begin{align*}
        \Delta^{\AM}_{i, j+1, \ell} &= \esp\left[ y_{T_s^{(2k)} + j + 1} - y_{T_s^{(2k)} + j + 2} \mid \ y_{T_s^{(2 k)}} = i, Z_{2 k} = \ell, s_{T_s^{(2 k)}} = \AM, y_{T_s^{(2 k)} + 1} = i + 1 \right] \\
& \quad \Pr\left[ y_{T_s^{(2 k)} + 1} = i + 1 \mid \ y_{T_s^{(2 k)}} = i, Z_{2 k} = \ell, s_{T_s^{(2 k)}} = \AM \right]\\
        &\quad + \esp\left[ y_{T_s^{(2k)} + j + 1} - y_{T_s^{(2k)} + j + 2} \mid \ y_{T_s^{(2 k)}} = i, Z_{2 k} = \ell, s_{T_s^{(2 k)}} = \AM, y_{T_s^{(2 k)} + 1} = i - 1 \right] \\
& \quad \Pr\left[ y_{T_s^{(2 k)} + 1} = i - 1 \mid \ y_{T_s^{(2 k)}} = i, Z_{2 k} = \ell, s_{T_s^{(2 k)}} = \AM \right] \\
        & = \left(\frac{n-i}{n} \right) \esp\left[ y_{T_s^{(2k)} + j} - y_{T_s^{(2k)} + j + 1} \mid \ Z_{2 k} = \ell, s_{T_s^{(2 k)}} = \AM, y_{T_s^{(2 k)} + 1} = i + 1 \right] \\
        &\quad + \left(\frac{i}{n} \right) \esp\left[ y_{T_s^{(2k)} + j} - y_{T_s^{(2k)} + j + 1} \mid \ Z_{2 k} = \ell, s_{T_s^{(2 k)}} = \AM, y_{T_s^{(2 k)} + 1} = i - 1 \right]  \numberthis\label{equality 3} \\
        & =\left(\frac{n-i}{n} \right) 
        \Delta^{\AM}_{i+1, j, \ell} + \left(\frac{i}{n} \right) \Delta^{\AM}_{i-1, j, \ell} \\
        &= \left( \frac{n-i}{n} \right) \left(\frac{n-2}{n}\right)^j \frac{2(i+1)-n}{n} + \left(\frac{i}{n} \right) \left(\frac{n-2}{n}\right)^j \frac{2(i-1)-n}{n} \numberthis\label{equality 4}\\
        &= \left(\frac{n-2}{n}\right)^{j+1} \frac{2i-n}{n},
    \end{align*}
    where in \eqnref{equality 3}, we the first step analysis on a phase of \AM. In \eqnref{equality 4}, we use the hypothesis of induction from \eqnref{lem:driftAM-HR}.

    Recall that $\Delta^{\AM}_{i, \ell} = \sum\limits_{j = 0}^{\ell - 1} \Delta^{\AM}_{i, j, \ell}$ and thus we obtain using \eqnref{lem:driftAM-HR}
    \begin{align*}
        \Delta^{\AM}_{i, \ell} &= \sum_{j = 0}^{\ell - 1} \left(\frac{n-2}{n}\right)^{j} \frac{2i-n}{n}\\
        &= \frac{1 - \left(\frac{n-2}{n}\right)^{\ell}}{1 - \frac{n-2}{n}} \frac{2i - n}{n} \\
        &= \left( 1 - \left(\frac{n-2}{n}\right)^{\ell} \right) \frac{2i - n}{2}.
    \end{align*}

    Now, we will use \eqnref{conditioningAM} to compute the drift on a phase of \AM as follows
    \begin{align*}
        \Delta^{\AM}_{i} &=  \sum_{\ell = 1}^{\infty}  \Delta^{\AM}_{i, \ell} (1 - q)^{\ell - 1} q \\
        &= \sum_{\ell = 1}^{\infty} \left( 1 - \left(\frac{n-2}{n}\right)^{\ell} \right) \frac{2i - n}{2} (1 - q)^{\ell - 1} q \\
        &= \frac{2i - n}{2} q \sum_{\ell = 1}^{\infty} \left( (1 - q)^{\ell - 1} - \frac{n-2}{n} \left( \frac{n-2}{n} (1 - q) \right)^{\ell - 1} \right) \\
        &= \frac{2i - n}{2} q  \left( \frac{1}{q} - \frac{n-2}{n} \frac{1}{1 - \frac{n-2}{n} (1-q)} \right)\\
        &= \frac{2 i - n}{2 + q (n-2)},
    \end{align*}
    as expected.
\end{proof}

\begin{proof}[Proof of \Cref{lem:driftOI}]
    Let $i \in [0..n]$, notice that the drift $\Delta^{\OI}_i$ can be written as
        \[ \Delta^{\OI}_i = i - \esp\left[ y_{T_s^{(1)}} \mid \, y_0 = i, s_0 = \OI \right], \]
    hence, we only need to compute the average number of zero bits $y_{T_s^{(1)}}$ after one phase of \OI. To ease the notation, let's define
        \[ \mathcal{Y}^{\OI}_i = \esp\left[ y_{T_s^{(1)}} \mid \, y_0 = i, s_0 = \OI \right], \]
    and 
        \[ \mathcal{Y}^{\OI}_{i, \ell} = \esp\left[ y_{\ell} \mid \, y_0 = i, Z_0 = \ell, s_0 = \OI \right] \]
    where $\ell \in \N^*$ denote the length of the phase of \OI. As for the drift $\Delta^{\AM}_i$ in \Cref{lem:driftAM}, the following decomposition according to the length $Z_0$ of the phase of \OI holds
        \begin{align*}
             \mathcal{Y}^{\OI}_i & = \sum_{\ell = 1}^{\infty}  \mathcal{Y}^{\OI}_{i, \ell} \Pr\left[ Z_0 = \ell \mid y_0 = i, s_0 = \OI \right] \\
            & = \sum_{\ell = 1}^{\infty} \mathcal{Y}^{\OI}_{i, \ell} \Pr\left[ Z_0 = \ell \mid s_0 = \OI \right] \numberthis\label{lem:driftOI-eq1} \\
            & = \sum_{\ell = 1}^{\infty} \mathcal{Y}^{\OI}_{i, \ell} (1 - p)^{\ell - 1} p,
        \end{align*}
    where in \eqnref{lem:driftOI-eq1} we use \Cref{lem:independance} stating the independence of $y_0$ and $Z_0$. Now, it remains to find a closed form for all the $\left( \mathcal{Y}^{\OI}_{i, \ell} \right)_{\ell \geq 1}$.

    First, when $\ell = 1$, we have, for any $0 \leq i \leq n$
        \begin{align*}
            \mathcal{Y}^{\OI}_{i, 1} & = \esp\left[ y_1 \mid \, y_0 = i, Z_0 = 1, s_0 = \OI \right] \\ 
            & = i \left(\frac{n - i}{n} \right) + (i - 1) \frac{i}{n} \\ 
            & = \frac{i}{n} \left( (n - i) + (i - 1) \right) \\ 
            & = i \left( 1 - \frac{1}{n} \right).
        \end{align*}

    We will now show by induction on $\ell \geq 1$ that for all $i \in [0.. n]$ we have 
        \[ \mathcal{Y}^{\OI}_{i, \ell} = i \left( 1 - \frac{1}{n} \right)^{\ell}. \numberthis\label{lem:driftOI-HR} \]

    We already initialized the above property so let's assume \eqnref{lem:driftOI-HR} holds for some $\ell \geq 1$. Then, either $i = 0$ and as the \OI operator cannot climb down \onemax
        \[ \mathcal{Y}^{\OI}_{0, \ell + 1} = \esp\left[ y_{\ell} \mid \, y_0 = 0, Z_0 = \ell, s_0 = \OI \right] = 0, \]
    hence \eqnref{lem:driftOI-HR} holds. Otherwise when $i \in [1.. n]$, by conditioning on the value of $y_1$, i.e., with a first step analysis on the phase of \OI, we obtain
        \begin{align*}
            \mathcal{Y}^{\OI}_{i, \ell + 1} & = \esp\left[ y_{\ell + 1} \mid \, y_0 = i, Z_0 = \ell + 1 + 1, s_0 = \OI \right] \\ 
            & = \Pr[y_1 = i \, \mid \, y_0 = i, Z_0 = \ell + 1, s_0 = \OI] \\
            &\qquad\qquad \times \esp\left[ y_{\ell + 1} \mid \, y_0 = i, y_1 = i, Z_0 = \ell + 1, s_0 = \OI \right] \\ 
            &\quad + \Pr[y_1 = i - 1 \, \mid \, y_0 = i, Z_0 = \ell + 1, s_0 = \OI] \\ 
            &\qquad\qquad \times \esp\left[ y_{\ell + 1} \mid \, y_0 = i, y_1 = i- 1, Z_0 = \ell + 1, s_0 = \OI \right] \\ 
            & = \frac{n - i}{n} \esp\left[ y_{\ell + 1} \mid \, y_1 = i, Z_0 = \ell + 1, s_0 = \OI \right] \numberthis\label{lem:driftOI-HR-eq1} \\
            & \quad + \frac{i}{n} \esp\left[ y_{\ell + 1} \mid \, y_1 = i - 1, Z_0 = \ell + 1, s_0 = \OI \right] \\ 
            & = \mathcal{Y}^{\OI}_{i, \ell} \left(  \frac{n - i}{n} \right) + \mathcal{Y}^{\OI}_{i - 1, \ell} \left( \frac{i}{n} \right) \\ 
            & = i \left( \frac{n - 1}{n} \right) \left( 1 - \frac{1}{n} \right)^{\ell} + (i - 1) \left( \frac{i}{n} \right) \left( 1 - \frac{1}{n} \right)^{\ell} \numberthis\label{lem:driftOI-HR-eq2} \\ 
            & = \frac{i}{n} \left[ (n - i) + (i - 1) \right] \left( 1 - \frac{1}{n} \right)^{\ell} \\ 
            & = i \left( 1 - \frac{1}{n} \right)^{\ell + 1},
        \end{align*}
    as desired. Thus, the formula \eqnref{lem:driftOI-HR} holds true for all $\ell \geq 1$ and $0 \leq i \leq n$. Above, in \eqnref{lem:driftOI-HR-eq1} we use the first step analysis on a phase of \OI and based on the event $\{ y_1 = i, Z_0 = \ell + 1, s_0 = \OI \}$ the expectation in \eqnref{lem:driftOI-HR-eq1} simplifies to
        \[ \esp\left[ y_{\ell + 1} \mid \, y_1 = i, Z_0 = \ell + 1, s_0 = \OI \right] = \esp\left[ y_{\ell} \mid \, y_0 = i, Z_0 = \ell, s_0 = \OI \right], \]
   since, as $Z_0 = \ell + 1 > 1$, we have $s_1 = \OI$ and we can forget the past, i.e., the event $\{ (y_0, s_0) = (i, \OI) \}$. In \eqnref{lem:driftOI-HR-eq2}, we use the induction hypothesis from \eqnref{lem:driftOI-HR}.

   Finally, it remains to plug the formula of $\mathcal{Y}^{\OI}_{i, \ell}$ in the sum \eqnref{lem:driftOI-eq1}, hence
    \begin{align*}
        \mathcal{Y}^{\OI}_i & = \sum_{\ell = 1}^{\infty} \mathcal{Y}^{\OI}_{i, \ell} (1 - p)^{\ell - 1} p \\ 
        & = \sum_{\ell = 1}^{\infty} \mathcal{Y}^{\OI}_{i, \ell} \left( i \left( 1 - \frac{1}{n} \right)^{\ell} (1 - p)^{\ell - 1} p \right) \\ 
        & = i p \left(1 - \frac{1}{n} \right) \sum_{\ell = 1}^{\infty} \left( \left(1 - \frac{1}{n} \right)^{\ell - 1} (1 - p)^{\ell - 1} \right) \\ 
        & = i p \left(1 - \frac{1}{n} \right) \sum_{\ell = 1}^{\infty} \left( \left(1 - \frac{1}{n} \right) (1 - p) \right)^{\ell - 1} \numberthis\label{lem:driftOI-final-eq1} \\ 
        & = i p \left(1 - \frac{1}{n} \right) \frac{1}{1 - (1 - p) \frac{n - 1}{n}} \\ 
        & = \frac{i p (n - 1)}{n - (1 - p) (n - 1)} \\ 
        & = \frac{i p (n - 1)}{1 + p (n - 1)},
    \end{align*}
where in \eqnref{lem:driftOI-final-eq1} we recognize a geometric series.

    Thus, the drift $\Delta^{\OI}_i$ is
        \[ \Delta^{\OI}_i = i - \mathcal{Y}^{\OI}_i = \frac{i}{1 + p (n - 1)} \]
    as expected.
\end{proof}

\begin{proof}[Proof of \Cref{lem:driftAM-OI}]
    Let $\Delta_i^{\AM-\OI} = \esp[y_0 - y_{T_s^{(2)}} \mid y_0 = i, s_0 = \AM]$, then
        \begin{align*}
            \Delta_i^{\AM-\OI} & = \esp[y_0 - y_{T_s^{(2)}} \mid y_0 = i, s_0 = \AM] \\ 
            & = \esp[y_0 - y_{T_s^{(1)}} \mid y_0 = i, s_0 = \AM] \\ & \qquad +  \esp[y_{T_s^{(1)}} - y_{T_s^{(2)}} \mid y_0 = i, s_0 = \AM] \\ 
            & = \Delta^{\AM}_i + \esp[y_{T_s^{(1)}} - y_{T_s^{(2)}} \mid y_0 = i, s_0 = \AM],
        \end{align*}
    where we split the initial drift between $y_0$ and $y_{T_s^{(2)}}$ in two drifts, one over the phase of \AM and the other over the phase of \OI (but the conditioning is different from the one in $\Delta^{\OI}_i$). Now, let $\esp_i[\cdot] = \esp[\cdot \mid y_0 = i, s_0 = \AM]$ and by the tower property of the expectation
        \begin{align*}
            \esp[y_{T_s^{(1)}} - y_{T_s^{(2)}} \mid & \, y_0 = i, s_0 = \AM] \\
            & = \esp_i[\esp_i[y_{T_s^{(1)}} - y_{T_s^{(2)}} \mid y_{T_s^{(1)}}]],
        \end{align*}
    where the inner expectation is 
        \begin{align*}
            \esp_i[& \, y_{T_s^{(1)}} - y_{T_s^{(2)}} \mid  y_{T_s^{(1)}}] \\ &=\esp[y_{T_s^{(1)}} - y_{T_s^{(2)}} \mid y_{T_s^{(1)}}, y_0 = i, s_0 = \AM, s_{T_s^{(1)}} = \OI] \\
            & = \esp[y_{T_s^{(1)}} - y_{T_s^{(2)}} \mid y_{T_s^{(1)}}, s_{T_s^{(1)}} = \OI] \\
            & = \Delta^{\OI}_{y_{T_s^{(1)}}}, 
        \end{align*}
    where we use \Cref{lem:independance} since $\onemax \in \seqopt_0$ to forget about the past event $\{y_0 = i, s_0 = \AM\}$ of the Markov chain $((y_t, s_t))_{t \in \N}$. Finally, using the fact that $\Delta^{\OI}_{i}$ is affine in $i$, the function $i \mapsto \Delta^{\OI}_i$ then commutes with the expectation thus
        \[ \esp_i[\Delta^{\OI}_{y_{T_s^{(1)}}}] = \Delta^{\OI}_{\esp_i[y_{T_s^{(1)}}]} = \Delta^{\OI}_{\left( i - \Delta^{\AM}_i \right)}, \]
    since 
        \[ \esp_i[y_{T_s^{(1)}}] = i - \esp[y_0 - y_{T_s^{(1)}} \mid y_0 = i, s_0 = \AM] = i - \Delta^{\AM}_i. \]
\end{proof}

\begin{proof}[Proof of \Cref{lem:avg-phases-AM}]
    Given the initial condition $(x_0, s_0) = (x, \AM)$ where $x \in \layer_m$ (a local maximum), we are interested in the probability $q^*$ that the trajectory $(x_0, \ldots, x_{T_s^{(1)}})$ of \AM reaches $x^*$. To lower bound $q^*$, we will only consider the trajectories that go straight right to layer $\layer_1$ and then to the global maximum in the next iteration (using \AM or \OI), i.e., trajectories for which $T_s^{(1)} \geq m - 1$ and $y_0 = m$, $y_1 = m - 1$, $\ldots$, $y_m = 0$. Hence, denoting $Y_m$ the event $\{y_1 = m - 1, \ldots, y_m = 0\}$.
        \begin{align*}
            q^* & \geq \Pr[T_s^{(1)} \geq m - 2, Y_m \mid y_0 = m, s_0 = \AM] \\ 
            & = \Pr[T_s^{(1)} \geq m - 2 \mid y_0 = m, s_0 = \AM] \numberthis\label{lem:avg-phases-AM-eq1} \\ 
            &\qquad \times \Pr[Y_m \mid y_0 = m, s_0 = \AM, T_s^{(1)} \geq m - 1] \\ 
            & = (1 - q)^{m - 2} \Pr[Y_m \mid y_0 = m, s_0 = \AM, T_s^{(1)} \geq m - 1] \\ 
            & = (1 - q)^{m - 2} \frac{m}{n} \frac{m - 1}{n} \cdots \frac{1}{n} \\ 
            & = (1 - q)^{m - 2} \frac{m!}{n^m},
        \end{align*}
    where, in \eqnref{lem:avg-phases-AM-eq1} we use the independence of $Z_0 = T_s^{(1)}$ and $y_0$ as stated in \Cref{lem:independance} which leads to
        \begin{align*}
            \Pr[T_s^{(1)} \geq m - 1 \mid & \, y_0 = m, s_0 = \AM]  \\ & =  \Pr[Z_0 \geq m - 1 \mid s_0 = \AM] \\ 
            & = \sum_{\ell = m - 1}^{\infty} (1 - q)^{\ell - 1} q \\ 
            & = (1 - q)^{m - 2}.
        \end{align*} 
\end{proof}

\begin{proof}[Proof of \Cref{lem:drift-gap-region}]
   Given any bit-string $x \in \{0,1\}^n$ such that $n - m < \| x \|_1 < n$, we have
   \begin{align*}
       \esp[d(x_t) - d(x_{t + 1}) & \mid x_t = x, s_t = \OI] \\
       & = \begin{cases}
            \frac{\| x \|_1}{n}, & \text{if $\| x \|_1 \neq n - 1$;} \\ 
            \frac{m - 1}{n} + \frac{n - 1}{n}, & \text{if $\| x \|_1 = n - 1$;}
        \end{cases}
   \end{align*}
    thus $\esp[d(x_t) - d(x_{t + 1}) \mid x_t = x, s_t = \OI] \geq \frac{\| x \|_1}{n}$ while, for the \AM operator
    \begin{align*}
        \esp[d(x_t) - d(x_{t + 1}) & \mid x_t = x, s_t = \AM] \\
        &= \begin{cases}
            \frac{2 \| x \|_1 - n}{n}, & \text{if $\| x \|_1 \neq n - 1$;} \\ 
            \frac{m - 1}{n} + \frac{n - 1}{n}, & \text{if $\| x \|_1 = n - 1$;}
        \end{cases}
    \end{align*}
    hence, as $\| x \|_1 < n$, this gives $\| x \|_1 \geq 2 \| x \|_1 - n $ and combining both drift for \OI and \AM we obtain the overall lower bound
    \begin{align*}
         \esp[d(x_t) - d(x_{t + 1}) \mid x_t = x] & \geq \frac{2 \| x \|_1 - n}{n} \\
         & = \frac{2 d(x) + n - 2 m}{n} \\
         & \geq \frac{2 d(x)}{n},
    \end{align*}
    since, in the region $n - m < \| x \|_1 < n$, we have $d(x) = \| x \|_1 - (n - m)$. 
\end{proof}

\begin{proof}[Proof of \Cref{lem:drift-left-slope}]
   First, as both \onemax and $\jump_m$ have a similar increasing slope (in the $1$-norm) on all bit-strings $x$ such that $\| x \|_1 \leq n - m$ then, the time $T_0 = \inf\{ t \geq 0 \mid \| x_t \|_1 = n - m\}$ (as defined on $\jump_m$) is the same as the time taken by the \MMAHH on \onemax starting initially with $(x_0, s)$ (recall that $s \in \{\OI, \AM\}$). Hence, it is enough to work on \onemax only, with the same initial conditions. Moreover, we have $T_0 \leq T_0^*$ where $T_0^*$ is the first time where the number of one-bits becomes greater or equal to $n - m$ at the end of a pair of phases \AM\!+\OI, i.e., 
        \[ T_0^* = \sum_{k = 0}^{k^* - 1}(Z_{2k} + Z_{2k+1}), \]
    where $k^* = \inf\{k \in \N \mid y_{T_s^{(2k)}} \leq m \}$. Note that, if $s = \OI$ then, up to an additional term of $\frac{1}{p}$ (we wait until the phase of \OI ends and either we already reach a local maximum -- in which case we stop -- or we continue with \AM), we can assume we start in $s = \AM$.
    
    Now, in order to have an upper bound on $\esp\left[k^*\right]$, we will use the additive drift theorem with overshooting (see \Cref{thm:add-drift-overshooting}), with the sequence $\left( y_{T_s^{(2 k)}} \right)_{0 \leq k \leq k^*}$ and the potential $z_t = y_t - m = (n - m) - \| x_t \|_1 \geq -m$ for which $\left( y_{T_s^{(2 k)}} \right)_{0 \leq k \leq k^*}$ is thus lower bounded. Moreover, with this potential, we have $k^* = \inf\{k \in \N, z_{T_s^{(2k)}} \leq 0 \}$ and, for any $0 \leq k < k^*$ and any $0 < i \leq n - m$, as we start with $s_0 = \AM$ then $s_{T_s^{(2 k)}} = \AM$ and using \Cref{lem:driftAM-OI} we have
    \begin{align*}
        \esp[ & z_{T_s^{(2k)}} - z_{T_s^{(2k + 2)}} \mid z_{T_s^{(2k)}} = i ] \\
        & = \Delta^{\AM-\OI}_{m + i} \\ 
        & =  \esp[ y_{T_s^{(2k)}} - y_{T_s^{(2k + 2)}} \mid y_{T_s^{(2k)}} = i + m ] \\
        & = \frac{(i+m)(2 + q(n-2) + 2 p (n-1)) - np(n-1)}{(1 + p(n-1))(2 + q(n-2))} \\
        & = \frac{i(2 + q(n-2) + 2 p (n-1))}{(1 + p(n-1))(2 + q(n-2))} \\
        &\quad + \frac{m(2 + 2 p (n-1) + q(n-2)) - np(n-1)}{(1 + p(n-1))(2 + q(n-2))}, \numberthis\label{lem:drift-left-slope-eq1}
    \end{align*}
    and $m(2 + 2 p (n-1) + q(n-2)) - np(n-1)$ from \eqnref{lem:drift-left-slope-eq1} is non-negative if, and only if, 
        \[ 2 m + m q (n - 2) \geq p (n - 1) (n - 2 m), \]
    i.e., either if $n \geq m \geq \frac{n}{2}$ or when $\frac{n}{2} > m \geq 0$ and
        \[ \frac{1}{n - 2 m}\left( \frac{m (n - 2)}{n - 1} q + \frac{2 m}{n - 1} \right) \geq p. \]
    Now, given $n \geq 3$ then $\frac{n - 2}{n - 1} \geq \frac{1}{2}$ and moreover, $\frac{2}{n - 1} \geq \frac{2}{n}$. Hence, for \eqnref{lem:drift-left-slope-eq1} to be non-negative, it suffices to have $n \geq 3$ and either $n \geq m \geq \frac{n}{2}$ or
        \[ \frac{n}{2} > m \geq 0 \, \text{ and } \, \frac{m}{2 (n - 2 m)}\left( q + \frac{4}{n} \right) \geq p. \]

    Now, we obtain
        \begin{align*}
            \Delta^{\AM-\OI}_{m + i} & \geq \frac{i(2 + q(n-2) + 2 p (n-1))}{(1 + p(n-1))(2 + q(n-2))} \\ 
            & \geq \frac{i(2 + q(n-2))}{(1 + p(n-1))(2 + q(n-2))} \\ 
            & =  \frac{i}{1 + p(n-1)} \\ 
            & \geq \frac{i}{1 + p n},
        \end{align*}
    and we can now apply \Cref{thm:add-drift-overshooting} using the fact that the overshooting is at most $\esp[ z_{T_s^{(2k^*)}} ] \geq - m$ since $y_{T_s^{(2k^*)}} \geq 0$ hence
        \[\esp[k^*] \leq (n - \| x_0 \|_1) (1 + pn), \]
    because the potential at the starting point is $z_0 = y_0 - m = n - \| x_0 \|_1 - m$ with an overshooting of $m$ hence, giving the factor $n - \| x_0 \|_1$.

    Finally, using Wald's theorem (\Cref{thm:wald}) we obtain the estimate
        \begin{align*}
            \esp\left[ T_0 \right] & \leq \frac{1}{p} + \esp[T_0^*] \\ 
            & =  \frac{1}{p} + \esp\left[ \sum_{k = 0}^{k^* - 1}(Z_{2k} + Z_{2k+1}) \right] \\ 
            & \leq  \frac{1}{p} + \esp[k^*] \left( \frac{1}{p} + \frac{1}{q} \right) \\ 
            & = \O\left( (n - \| x_0 \|_1)  (1 + pn) \left(\frac{1}{p} + \frac{1}{q} \right) \right).
        \end{align*}
       
    For the particular case $\| x_0 \|_1 = n - m - 1$, we have
    \begin{align*}
        \esp\left[ T_0 \right] & \leq \frac{1}{p} + (n - \| x_0 \|_1)  (1 + pn) \left(\frac{1}{p} + \frac{1}{q} \right) \\ 
        & = \frac{1}{p} + (m + 1) (1 + pn) \left(\frac{1}{p} + \frac{1}{q} \right) \\ 
        & = \O\left( m (1 + pn) \left(\frac{1}{p} + \frac{1}{q} \right) \right),
    \end{align*}
    as desired.
\end{proof}

\begin{proof}[Proof of \Cref{lem:avg-time-to-climb-jump}]
    As the \MMAHH relies on the \textsc{RandomOneBitFlip} mutation operator, it is enough for upper bounding $\esp[T_1]$ to distinguish between $\| x_0 \|_1 < n - m$ and $n - m < \| x_0 \|_1 < n$ then, use the drift computed in \Cref{lem:drift-gap-region} and \Cref{lem:drift-left-slope} and apply separately one of the drift theorems from section \ref{subsec:math-tools}.

    First, if $n - m < \| x_0 \|_1 < n$ then, for any time $0 \leq t < T_1$, we still have $n - m < \| x_t \|_1 < n$ and by \Cref{lem:drift-gap-region} we can apply the multiplicative drift with $\delta = \frac{2}{n}$ in the gap region (at the scale of the bit-string) hence for any $x\in \{0, 1\}^n$ such that $n - m < \| x \|_1 < n$ we have
    \begin{align*}
        \esp[T_1 \mid x_0 = x] & \leq \frac{1 + \log(d(x))}{\delta} \\ 
        & \leq \frac{1 + \log(m)}{\delta} \\
        &= O(n \log(m)).
    \end{align*}

    Now, consider the case where $0 \leq \| x_0 \|_1 < n - m$, as of before, whatever the time $0 \leq t < T_1$ we have $0 \leq \| x_t \|_1 < n - m$. Moreover, as $s_0 = \OI$, by simply waiting until the phase of \OI ends (and either we already reached a local maximum -- in which case we stop -- or we continue with now \AM), we can assume we start in \AM. Now by \Cref{lem:drift-left-slope} we obtain, for any $x\in \{0, 1\}^n$ such that $n - m < \| x \|_1 < n$,
        \begin{align*}
            \esp[T_1 \mid x_0 = x] & \leq \frac{1}{p} + (n - \| x \|_1)  (1 + pn) \left(\frac{1}{p} + \frac{1}{q} \right) \\ 
            & = \O\left( (n - \| x \|_1)  (1 + pn) \left(\frac{1}{p} + \frac{1}{q} \right) \right) \\ 
            & = \O\left( n (1 + pn) \left(\frac{1}{p} + \frac{1}{q} \right) \right).
        \end{align*}

    Of course, if $x_0 \in X^*$ then $T_1 = 0$. Hence, combining all these three cases leads to 
        \[ \esp[T_1] = \O\left( n (1 + p n) \left( \frac{1}{p} + \frac{1}{q} \right) \right), \]
    since $\log(m) = O(n) = \O\left( (1 + p n) \left( \frac{1}{p} + \frac{1}{q} \right) \right)$ and we are done.
\end{proof}

\end{document}